\documentclass[letterpaper]{article} % DO NOT CHANGE THIS
\usepackage{aaai25}  % DO NOT CHANGE THIS
\usepackage{times}  % DO NOT CHANGE THIS
\usepackage{helvet}  % DO NOT CHANGE THIS
\usepackage{courier}  % DO NOT CHANGE THIS
\usepackage[hyphens]{url}  % DO NOT CHANGE THIS
\usepackage{graphicx} % DO NOT CHANGE THIS
\usepackage{natbib}  % DO NOT CHANGE THIS AND DO NOT ADD ANY OPTIONS TO IT
\usepackage{caption} % DO NOT CHANGE THIS AND DO NOT ADD ANY OPTIONS TO IT
\usepackage{algorithm}
\usepackage{algpseudocode}
\usepackage{comment}
\usepackage{color}
\usepackage{xcolor}
\usepackage{multirow}

\usepackage{newfloat}
\usepackage{listings}

\usepackage{booktabs} % For formal tables
\usepackage{subcaption}
\usepackage{booktabs}
\usepackage{tabularx}
\usepackage{chngcntr}
\usepackage{amsmath, mathtools}
\usepackage{amsthm}
\usepackage{amsfonts}
\usepackage{tabularx}
\usepackage{thmtools}
\usepackage{thm-restate}
\usepackage{dsfont}

\usepackage{cleveref}

\DeclareMathOperator\supp{supp}
\DeclareMathOperator{\sign}{sign}
\newtheorem{example}{Example}%[section]
\newtheorem{definition}{Definition}%[section]
\newtheorem{assumption}{Assumption}%[section]%\setcounter{assumption}{0}

\usepackage{thmtools}
\usepackage{thm-restate}
\usepackage{newfloat}
\usepackage{listings}
\DeclareCaptionStyle{ruled}{labelfont=normalfont,labelsep=colon,strut=off} % DO NOT CHANGE THIS
\floatstyle{ruled}
\newfloat{listing}{tb}{lst}{}
\floatname{listing}{Listing}
\title{MCGAN: Enhancing GAN Training with\\ Regression-Based Generator Loss}
\author {
    Baoren Xiao\textsuperscript{\rm 1},
    Hao Ni\textsuperscript{\rm 1},
    Weixin Yang\textsuperscript{\rm 2}\thanks{Corresponding author.}
}
\affiliations {
    \textsuperscript{\rm 1}University College London\\
    \textsuperscript{\rm 2}University of Oxford\\
    baoren.xiao.18@ucl.ac.uk,
    h.ni@ucl.ac.uk, 
    wxy1290g@gmail.com
}

\usepackage{bibentry}
\begin{document}

\maketitle

\begin{abstract}
Generative adversarial networks (GANs) have emerged as a powerful tool for generating high-fidelity data. However, the main bottleneck of existing approaches is the lack of supervision on the generator training, which often results in undamped oscillation and unsatisfactory performance. To address this issue, we propose an algorithm called Monte Carlo GAN (MCGAN). This approach, utilizing an innovative generative loss function, termly the regression loss, reformulates the generator training as a regression task and enables the generator training by minimizing the mean squared error between the discriminator's output of real data and the expected discriminator of fake data. 
We demonstrate the desirable analytic properties of the regression loss, including discriminability and optimality, and show that our method requires a weaker condition on the discriminator for effective generator training. These properties justify the strength of this approach to improve the training stability while retaining the optimality of GAN by leveraging strong supervision of the regression loss. Extensive experiments on diverse datasets, including image data (CIFAR-10/100, FFHQ256, ImageNet, and LSUN Bedroom), time series data (VAR and stock data) and video data, are conducted to demonstrate the flexibility and effectiveness of our proposed MCGAN. Numerical results show that the proposed MCGAN is versatile in enhancing a variety of backbone GAN models and achieves consistent and significant improvement in terms of quality, accuracy, training stability, and learned latent space.  
%Numerical results on CIFAR-10 and CIFAR-100 datasets demonstrate that the proposed MCGAN significantly and consistently improves the existing state-of-the-art GAN models in terms of quality, accuracy, training stability, and learned latent space. Furthermore, the proposed algorithm exhibits great flexibility for integrating with a variety of backbone models to generate spatial images, temporal time-series, and spatio-temporal video data. 
\end{abstract}

\begin{links}
\link{Code}{https://github.com/DeepIntoStreams/MCGAN}
\end{links}

\section{Introduction}\label{sec:1}
In recent years, Generative  Adversarial  Network  (GAN) \cite{goodfellow2014generative} has become one of the most powerful tools for realistic image synthesis. However, the instability of the GAN training and unsatisfying performance remains a challenge. To combat it, much effort has been put into developing regularization methods, see \cite{gulrajani2017improved,mescheder2018training,miyato2018spectral,kang2022studiogan}. Additionally, as \cite{arjovsky2017towards} pointed out, the generator usually suffers gradient vanishing and instability due to the singularity of the denominator showed in the gradient when the discriminator becomes accurate. To address this issue, some work has been done to develop better adversarial loss \cite{lim2017geometric,mao2017least,arjovsky2017wasserstein}. As a variant of GAN, conditional GAN (cGAN) \cite{mirza2014conditional} is designed to learn the conditional distribution of target variable given conditioning information. It improves the GAN performance by incorporating conditional information to both the discriminator and generator, we hence have better control over the generated samples \cite{zhou2021omni,odena2017conditional}.
  
Unlike these works on the regularization method and adversarial loss, our work focuses on the generative loss function to enhance the performance of GAN training. In this paper, we propose a novel generative loss, termed as the regression loss $\mathcal{L}_{R}$, which reformulates the generator training as the least-square optimization task. This regression loss underpins our proposed MCGAN, an enhancement of existing GAN models achieved by replacing the original generative loss with our regression loss. This approach leverages the expected discriminator $D^{\phi}$ under the fake measure induced by the generator. Benefiting from the strong supervision lying in the regression loss, our approach enables the generator to learn the target distribution with a relatively weak discriminator in a more efficient and stable manner. 

%novel algorithm called \emph{Monte-Carlo Generative Adversarial Network} (MCGAN) that has the novel regression loss as the generative loss function incorporating the Monte-Carlo estimator of fake generator incorporates \emph{Mean Square Error} (MSE) and \emph{Monte-Carlo} method into the generative loss function.  Benefiting from the strong supervision lying in the regression loss, this innovative generative loss can train the generator to learn the target distribution with a relatively weak discriminator. 
The main contributions of our paper are three folds:
\begin{itemize}
    \item We propose the MCGAN methodology for enhancing both unconditional and conditional GAN training.
    \item We establish the theoretical foundation of the proposed regression loss, e.g., the discriminability, optimality, and improved training stability.  A simple but effective toy example of Dirac-GAN is provided to show that our proposed MCGAN successfully mitigates the non-convergence issues of conventional GANs by incorporating regression loss.
    \item We empirically validate the consistent improvements of MCGAN over various GANs across diverse data types (i.e., images, time series, and videos). Our approach improves quality, accuracy, training stability, and learned latent space, showing its generality and flexibility.   
\end{itemize}
%On the one hand, we establish the theoretical foundation of the proposed regression loss, such as the discriminivity and optimality. Besides, we show the improved stability of the proposed MCGAN. On the other hand, we provide a simple but clear example of Dirac GAN, in which our proposed MCGAN successfully mitigate the issues of non-convergence of the conventional GANs by incorporating the regression loss. Additionally,  we empirically validate that MCGAN considerably enhances the performance of the generator combined with commonly used discriminative losses like \emph{Hing loss} and \emph{Non-saturating loss} on the CIFAR-10 and CIFAR-100 datasets in terms of both training stability and accuracy. 
%\input{samplebody-journals}

%\noteNi{What are key contributions? (1) Reformulating generator training as a supervised learning problem where the model is obtained by the MC estimator of the generator, leading to a better generative loss; (give a new name for this new generative loss $\mathcal{L}_{R}$) (2) strength in conditional GANs; (3) establish theoretical results in terms of optimality of $\mathcal{L}_{R}$; (4) numerical validation on conditional image generation and video generation. }
\paragraph{Related work} GANs have demonstrated their capacity to simulate high-fidelity synthetic data, facilitating data sharing and augmentation. Extensive research has focused on designing GAN models for various data types, including images \cite{han2018gan}, time series \cite{yoon2019time,xu2020cot,ni2021sig}, and videos \cite{gupta2022rv}. Recently, Conditional GANs (cGANs) have gained significant attention for their ability to generate synthetic data by incorporating auxiliary information \cite{yoon2019time,liao2024sig, xu2019modeling}. For the integer-valued conditioning variable, conditional GANs can be roughly divided into two groups depending on the way of incorporating the class information: \textit{Classification-based}  and \textit{Projection-based} cGANs \cite{odena2017conditional,miyato2018cgans,kang2021rebooting,zhou2021omni,mirza2014conditional,hou2022conditional}. For the case where conditioning variable is continuous, the training of conditioning GANs is more challenging. For example, conditional WGAN suffers difficulty in estimating the conditional expected discriminator of real data due to the need for recalibration per every discriminator update \cite{liao2024sig}. Attempts are made to mitigate this issue, such as conditional SigWGAN \cite{liao2024sig}, which is designed to tackle this issue for time series data.

\section{Preliminaries}
\subsection{Generative adversarial networks}
Generative adversarial networks (GANs) are powerful tools for learning the target distribution from real data to enable the simulation of synthetic data. To this goal, GAN plays a min-max game between two networks: \textit{Generator} ($G$) and \textit{Discriminator} ($D$). Let $\mathcal{X}$ denote the target space and $\mathcal{Z}$ be the latent space. Then \textit{Generator} $G^{\theta}$ is defined as a parameterised function that maps latent noise $z\in\mathcal{Z}$ to the target data $x\in\mathcal{X}$, where $\theta \in \Theta $ is the model parameter of $G$. \textit{Discriminator} $D^{\phi}: \mathcal{X} \to \mathbb{R}$ discriminates between the real data and fake data generated by the generator. 

%Generative adversarial network (GAN) \cite{goodfellow2014generative} aims to learn the target distribution from real data. To this goal, GAN plays a min-max game between two networks: \textit{Generator} and \textit{Discriminator}. Let $\mathcal{X}$ denote the target space and $\mathcal{Z}$ be the latent space. Then \textit{Generator}, denoted by $G$ is defined as a parameterised function $G$ that maps latent noise $z\in\mathcal{Z}$ to the target data $x\in\mathcal{X}$, i.e. $G:\Theta\times \mathcal{Z} \to \mathcal{X}$.  \textit{Discriminator} denoted by $D$ discriminates between the real data and fake data generated by the generator, it is defined as a parameterised function that maps the data into real space $\mathbb{R}$, i.e.,  $D: \phi\times \mathcal{X} \to \mathbb{R}$.

Let $\mu$ and $\nu_{\theta}$ denote the true measure and fake measure induced by $G^{\theta}$. %Two measures $\mu$ and $\nu_{\theta}$ are usually assumed to be continuous and hence have density functions $p_{\mu}$ and $p_{\nu_{\theta}}$
For generality, the objective functions of GANs can be written in the following general form: 
\begin{align}\label{eq:ganobj}
&\max_{\phi}\mathcal{L}_D(\phi;\theta)=\mathbb{E}_{\mu}\left[f_1(D^\phi(X))\right]+\mathbb{E}_{\nu_\theta}\left[f_2(D^\phi(X))\right],\notag\\
&\min_{\theta}\mathcal{L}_G(\theta;\phi)=\mathbb{E}_{\nu_\theta}\left[h(D^\phi(X))\right],
\end{align}
where $f_1$, $f_2$ and $h$ are real-valued functions. Different choices of $f_1$, $f_2$ and $h$ lead to different GAN models. 

There are extensive studies concerned with how to measure the divergence or distance between $\mu$ and $\nu_\theta$ as the improved GAN loss function, which are instrumental in stabilising the training and enhancing the generation performance. Examples include \emph{Hinge loss} \cite{lim2017geometric},  \emph{Wasserstein loss} \cite{arjovsky2017wasserstein}, \emph{Least squares loss} \cite{mao2017least}, \emph{Energy-based loss} \cite{zhao2016energy} among others. Many of them satisfy Eqn. \eqref{eq:ganobj}.
\begin{example}\label{example_GAN}
    \begin{itemize}
        \item classical GAN \cite{goodfellow2014generative}: $f_1(w)=\log(w)$ and $f_2(w)=-h(w)=\log(1-w)$.
        \item HingeGAN \cite{lim2017geometric}: $f_1(w)=f_2(-w) = -\max(0,1-w)$, and $h(w) =-w$.
        \item Wasserstein GAN \cite{arjovsky2017wasserstein} : $f_1(w)=f_2(-w) = w$, and $h(w) =-w+c_{\mu}$, where $c_{\mu}:=\mathbb{E}_{X \sim \mu} [D^{\phi}(X)]$.        
    \end{itemize}
\end{example}

The Wasserstein distance is linked with the mean discrepancy. More specifically, let $d_{\phi}(\mu, \nu)$ denote the mean discrepancy between any two distributions $\mu$ and $\nu$ associated with test function $D^{\phi}$ defined as $d_{\phi}(\mu, \nu) = \mathbb{E}_{X \sim \mu} [D^{\phi}(X)] - \mathbb{E}_{X \sim \nu} [D^{\phi}(X)]$. In this case, $\mathcal{L}_{G}(\theta; \phi)$ could be interpreted as $d_{\phi}(\mu, \nu_{\theta})$.

\subsection{Conditional GANs}
Conditional GAN (cGAN) is a conditional version of a generative adversarial network that can incorporate additional information, such as data labels or other types of auxiliary data into both the generator and discriminative loss \cite{mirza2014conditional}. The goal of conditional GAN is to learn the conditional distribution $\mu$ of the target data distribution $X\in\mathcal{X}$ (i.e., image ) given the conditioning variable (i.e., image class label) $Y\in\mathcal{Y}$. More specifically, under the real measure $\mu$, $X \times Y$ denote the random variable taking values in the space $\mathcal{X} \times \mathcal{Y}$. The marginal law of $X$ and $Y$ are denoted by $P_X$ and $P_Y$, respectively. 

The conditional generator $G^{\theta}: \mathcal{Y} \times \mathcal{Z} \rightarrow \mathcal{X}$ incorporates the additional conditioning variable to the noise input, and outputs the target variable in $\mathcal{X}$. Given the noise distribution $Z$, $G^{\theta}(y)$ induces the fake measure denoted by $\nu_{\theta}(y)$, which aims to approximate the conditional law of $\mu(y):=P(X | Y=y)$ under real measure $\mu$. The task of training an optimal conditional generator is formulated as the following min-max game:
%To this goal, both discriminator and generator can be extended to incorporate the conditional information $Y$, i.e. $D:\Phi\times \mathcal{X}\times \mathcal{Y} \to \mathbb{R},$ and $G:\Theta\times\mathcal{Y}\times \mathcal{Z}  \to \mathcal{X}$ respectively, and both are trained to optimise the following objective function: ~

\begin{align} \label{orign_gloss}
&\mathcal{L}_D(\phi, \theta)=\mathbb{E}_{Y}\left[\mathbb{E}_{\mu(y)}[f_1(D^\phi(X))]+\mathbb{E}_{\nu_\theta(y)}[f_2( D^\phi(X)]\right],\notag\\
&\mathcal{L}_G(\theta;\phi)= \mathbb{E}_{Y}\left[\mathbb{E}_{\nu_\theta(y)}[h(D^\phi(X))]\right],
\end{align}
where $f_1$, $f_{2}$ and $h$ are real value functions as before and $\mathbb{E}_{Y}$ denotes that the expectation is taken over $y$ sampled from $Y$. Different from the unconditional case, $\mathcal{L}_D$ and $\mathcal{L}_G$ has in the outer expectation $\mathbb{E}_{y \sim P_{Y}}$ due to $Y$ being a random variable.

%where $\mathbb{E}_{\nu_{\theta}(Y)}[f(X)]$ means $\mathbb{E}_{y \sim P_{Y}, x \sim \nu_{\theta}(y)}[f(x)]]$. 

\section{Monte-Carlo GAN}
\subsection{Methodology}
%Suppose that we observe the dataset $\mathbf{X}=(X_i)_{i = 1}^{N}$, which is composed of the iid samples. 
In this section, we propose the Monte-Carlo GAN (MCGAN) for both unconditional and conditional data generation. Without loss of generality, we describe our methodology in the setting of the conditional GAN task.\footnote{The unconditional GAN can be viewed as the conditioning variable is set to be the empty set.} %Under the real measure $\mu$, $X \times Y$ denote the random variable taking values in the space $\mathcal{X} \times \mathcal{Y}$, where $X$ is target data distribution and $Y$ is the corresponding conditioning variable. 
%Let $G^{\theta}: \mathcal{Y} \times \mathcal{Z} \rightarrow \mathcal{X}$ denote a generator to map the noise space $\mathcal{Z}$ to the target space $X$ with the induced fake measure $\nu_{\theta}$. The discriminator: $D^{\phi}: \mathcal{X} \rightarrow \mathbb{R}$. 
Consider the general conditional GAN model composed with the generator loss $\mathcal{L}_{G}$ (Eqn. \eqref{orign_gloss}) and the discrimination loss $\mathcal{L}_{D}$ outlined in the last subsection. To further enhance the GAN model, we propose the following MCGAN by replacing the generative loss $\mathcal{L}_{G}$  with the following novel regression loss for training the generator from the perspective of the regression, denoted by $\mathcal{L}_{R}$, 
  \begin{eqnarray}\label{eq:unobj}      \mathcal{L}_R(\theta; \phi):=\mathbb{E}_{(x, y) \sim \mu}\left[|D^\phi(x)- \mathbb{E}_{\hat{x} \sim \nu_{\theta}(y)}[D^\phi(\hat{x})]|^2\right],
    \end{eqnarray}
where the expectation is taken under the joint law $\mu$ of $X$ and $Y$.
We optimize the generator's parameters $\theta$ by minimizing the regression loss $\mathcal{L}_R(\theta; \phi)$. We keep the discriminator loss and conduct the min-max training as before. The training algorithm of MCGAN is given in Algorithm \ref{algo:mcgan} in Appendix.
%\ref{appendix:magan_algo}.
%Like other GAN models, we optimize the discriminator's parameters $\phi$ by maximising the discriminator loss $\mathcal{L}$.  

The name for Monte Carlo in MCGAN is due to the usage of the Monte Carlo estimator of expected discriminator output under the fake measure. This innovative loss function reframes the conventional generator training into a mean-square optimization problem by computing the $l^{2}$ loss between real and expected fake discriminator outputs. %The detailed algorithm of MCGAN is provided in Algorithm \ref{algo:mcgan}.
 
Next, we explain the intuition behind $\mathcal{L}_{G}$ and its link with optimality of conditional expectation. Let us consider a slightly more general optimization problem for $\mathcal{L}_{R}$:
\begin{eqnarray}\label{l2_loss}
  \min_{f \in \mathcal{C}(\mathcal{Y}, \mathbb{R})} \mathbb{E}_{\mu}[|D^\phi(X) - f(Y)|^2 ],
\end{eqnarray}
It is well known that the conditional expectation is the optimal $l^2$ estimator. Therefore, the \textbf{minimizer} to Eqn \eqref{l2_loss} is given by the conditional expectation function $f^{*}: \mathcal{Y} \rightarrow \mathbb{R}$, defined as
\begin{eqnarray*}
f^{*}(y) = \mathbb{E}_{\mu}[D^{\phi}(X) | Y=y].
\end{eqnarray*}
This fact motivates us to consider the conditional expectation under the fake measure, $\mathbb{E}_{\nu_{\theta}(Y)}[D^\phi(X)]$, as the model for the mean equation $f^{*}$. It leads to our regression loss $\mathcal{L}_{R}$, where we replace $f$ by $\mathbb{E}_{\nu_{\theta}(Y)}[D^\phi(X)]$ in Eqn. \eqref{l2_loss}.

Minimising the regression loss $\mathcal{L}_{G}$ enforces the conditional expectation of $D^{\phi}(X)$ under fake measure $\nu_{\theta}(Y)$ to approach that under the conditional true distribution $\mu(Y)=\mathbb{P}(X | Y)$ for any given $D^{\phi}$.  Suppose that $(G^{\theta})_{{\theta} \in {\Theta}}$ provides a rich enough family of distributions containing the real distribution $\mu$. Then there exists $\theta^{*} \in \Theta$, which is a minimizer of $\mathcal{L}_R(\theta, \phi)$ for all possible discriminator's parameter $\phi$, and it satisfies that %$d_{\phi}(\mu(Y), \nu_{\theta^*}) = 0$.
\begin{eqnarray}
\mathbb{E}_{\mu(Y)}[D^{\phi}(X)] = \mathbb{E}_{\nu_{\theta^*}(Y)}[D^{\phi}(X)].
\end{eqnarray}
It implies that no matter whether the discriminator $D^\phi$ achieves the equilibrium of GAN training, the regression loss $\mathcal{L}_{R}$ is a valid loss to optimize the generator to match its expectation of $D^{\phi}$ between true and fake measure.

Moreover, we highlight that our proposed regression loss can effectively mitigate the challenge of the conditional Wassaserstain GAN (c-WGAN). To compute the generative loss of c-WGAN, one needs to estimate the conditional expectation $\mathbb{E}_{\mu(Y)}[D^{\phi}(X)]$. However, when the conditioning variable is continuous, estimating this conditional expectation becomes computationally expensive or even infeasible due to the need for recalibration with each discriminator update. In contrast, our regression loss does not need the estimator for $\mathbb{E}_{\mu(Y)}[D^{\phi}(X)]$.

%where $\mathcal{L}_{G}(\theta; \phi)=d_{\phi}(\mu(Y), G^{\theta}(y))$. It is challenging to obtain the accurate estimator for $\mathbb{E}_{\mu(Y)}[D^{\phi}(X)]$ for the case $Y$ is a continuous variable. Under this case, there might be only one realization of $X$ given the  

%\noteNi{connect it with conditional cWGAN to solve the bottleneck issue of conditional cWGAN.}

%\noteNi{MCGAN algorithm}
\subsection{Comparison between $\mathcal{L}_{R}$ and $\mathcal{L}_{G}$}

In this subsection, we delve into the training algorithm of the regression loss $\mathcal{L}_{R}$ and illustrate its advantages of enhancing the training stability in comparison with the generator loss $\mathcal{L}_{G}$. For ease of notation, we consider the unconditional case. To optimize the generator's parameters $\theta$ in our MCGAN, we apply gradient-descent-based algorithms and the updating rule of $\theta_n$ is given by  
\begin{align}\label{eq:gradientMCGAN}
\theta_{n+1}=&\theta_{n}-\lambda\frac{\partial \mathcal{L}_{R}}{\partial\theta}|_{\theta=\theta_n}\\
=&\theta_n-2\lambda \underbrace{\left(\mathbb{E}_{\mu}[D^{\phi}(X)]-\mathbb{E}_{\nu_{\theta_n}}[D^{\phi}(X)]\right)}_{d_{\phi}(\mu, \nu_{\theta_n})} H(\theta_n,\phi), \notag
\end{align}
where $\lambda$ is the learning rate and
\begin{eqnarray}\label{eqn_H}
H(\theta,\phi)={\mathbb{E}}_{z\sim P_Z}[\nabla_{\theta}G^{\theta}(z)^T\cdot \nabla_x D^\phi(G^{\theta}(z))].
\end{eqnarray}

Note the gradient $\frac{\partial \mathcal{L}_{R}}{\partial \theta}$ takes into account not only  $\nabla_x D^\phi(x)$ but also $d(\mu, \nu_{\theta})$ - the discrepancy between the expected discriminator outputs under two measures $\mu$ and $\nu_\theta$.

In contrast, employing the generator loss $\mathcal{L}_{G}$, the generator parameter $\theta$ is updated by the following formula:

\begin{align}\label{eq:f_G_update}
\theta_{n+1} = & \theta_n - \lambda {\mathbb{E}}_{z\sim P_Z}\Big[ h'(D^\phi(G^{\theta_n}(z)))  \nabla_{\theta} G^{\theta}(z)^T \Big|_{\theta=\theta_n} \notag \\
& \cdot \nabla_x D^\phi(G^{\theta_n}(z)) \Big]. 
\end{align}

One can see that Eqn. \eqref{eq:f_G_update} depends on the discriminator gradients $\nabla_x D^\phi(G^{\theta_n}(z))$ heavily. 

MCGAN benefits from the strong supervision of $\mathcal{L}_{R}$, which provides more control over the gradient behaviour during the training. When $\theta$ is close to the optimal $\theta^{*}$, even if $D^{\phi}$ is away from the optimal discriminator, $d_{\phi}(\mu, \nu_{\theta})$ would be small and hence leads to stabilize the generator training. However, it may not be the case for the generator loss as shown in Eq. \eqref{eq:f_G_update}, resulting in the instability of generator training. For example, this issue is evident for the Hinge loss where $h(x) = x$ as shown in \cite{mescheder2018training}.

\subsection{Illustrative Dirac-GAN example}
To illustrate the advantages of MCGAN, we present a toy example from \cite{mescheder2018training}, 
demonstrating its resolution of the training instability in Dirac-GAN. The Dirac-GAN example involves a true data distribution that is a Dirac distribution concentrated at 0. Besides, the Dirac-GAN model consists of a generator with a fake distribution $\nu_\theta(x) = \delta(x - \theta)$ with  $\delta(\cdot)$ is a Dirac function and a discriminator $D^{\phi}(x) = \phi x$.

We consider three different loss functions for both $\mathcal{L}_{D}$ and $\mathcal{L}_{G}$: (1) \emph{binary cross-entropy loss} (BCE), (2) \emph{Non-saturating loss} and (3) \emph{Hinge loss}, resulting GAN, NSGAN and HingeGAN, respectively. In this case, the unique equilibrium point of the above GAN training objectives is given by $\phi=\theta=0$. %However, it is shown in \cite{mescheder2018training} that they fail to converge Clearly, the optimal generator parameter $\theta$ is 0.

In this case, the updating scheme of training GAN is simplified to
$$
\begin{cases}
\phi_{n+1}=\phi_{n}+\lambda f'(-\phi_n\theta_n)\theta_n,\\
    \theta_{n+1}=\theta_{n}-\lambda h'(\phi_n\theta_n)\phi_{n}.
\end{cases}
$$
where $f$ is specified as $f(x) = -\log(1 + \exp(x))$.
By applying MCGAN to enhance GAN training, the update rules for the model parameters $\theta$ and $\phi$ are modified as follows:
$$
\begin{cases}
\phi_{n+1}&=\phi_{n}+\lambda f'(\phi_n\theta_n)\theta_n,\\
\theta_{n+1}&=\theta_{n}-\lambda2(\phi_n\theta_n-\phi_nc)\phi_{n}.
\end{cases}
$$

Fig. \ref{fig:diracgan} (a-c) demonstrates that GAN, NSGAN and Hinge GAN all fail to converge to obtain the optimal generator parameter $\theta^{*} = 0$. That is because the updating scheme of $\theta$ depends heavily on the $\phi$. When $\phi$ fails to converge to zero, $\theta$ continues to update even if it has reached zero, and the non-zero $\theta$ further encourages $\phi$ updating away from $0$, which results in a vicious cycle and the failure of both generator and discriminator. In contrast, Fig. \ref{fig:diracgan}(d) of MCGAN training demonstrates that the generator parameter $\theta$ successfully converges to the optimal value $0$ thanks to the regression loss in \eqref{eq:unobj} to bring the training stability of the generator. A 2D Gaussian mixture example is also provided in Appendix, showing that MCGAN can help mitigate model collapse. %hat MCGAN takes advantage of less dependency of $\theta$ on $\phi$.
\begin{figure}[ht]
    \includegraphics[width=0.472\textwidth]{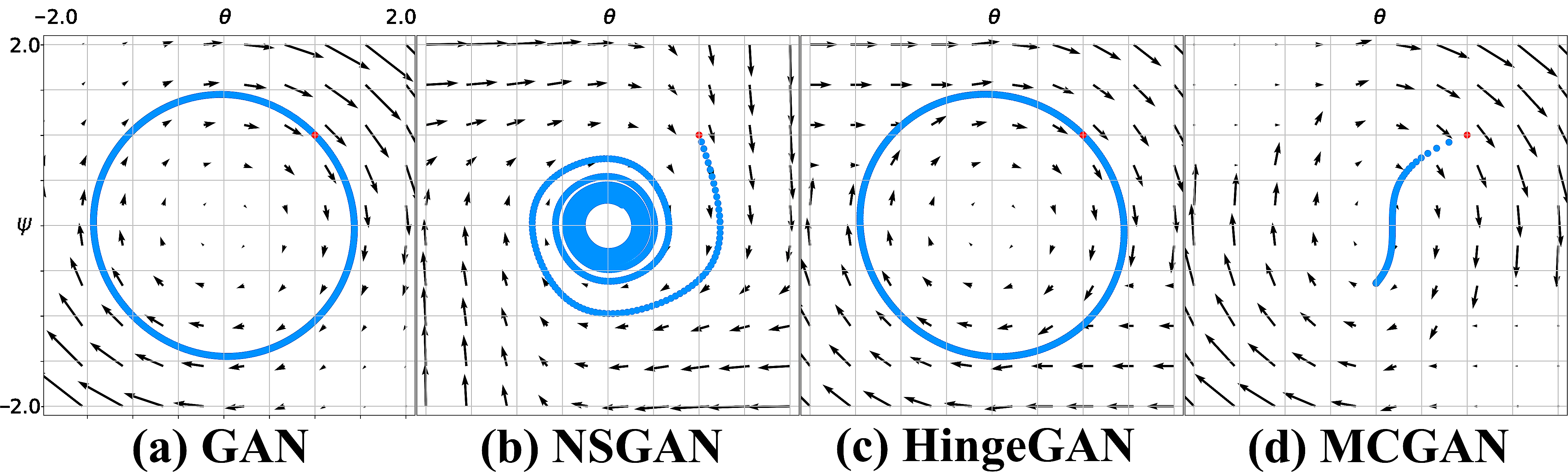}
\caption{Dirac-GAN example}
\label{fig:diracgan}
\end{figure}

\subsection{Discriminability and optimality of MCGAN}
To ensure that MCGAN training leads to the optimal generator $\nu_{\theta^{*}}=\mu$, one needs the sufficient discriminative power of $D^{\phi}$. The discriminative power of $D^{\phi}$ is determined by the discriminative loss function $\mathcal{L}_{D}$, which is usually defined as certain divergences, such as JS divergence in GAN \cite{goodfellow2014generative}. However, computing such divergence involves finding the optimal discriminator that optimizes the objective function, which might be challenging in practice. See \cite{liu2017approximation} for a comprehensive description of the discriminative loss function. 

Instead of requiring an optimal discriminator, we introduce the weaker condition, the so-called \emph{discriminability} of the discriminator $D^\phi$ to ensure the optimality of the generator for the MCGAN training. %In Theorem \ref{theorem:g}, we prove that the convergence of the optimal generator of MCGAN can be achieved under a weaker condition of the discriminator due to the strong supervision of the regression loss.

%Note that \eqref{eq:match} would not necessarily lead to $\nu_{\theta^{*}}=\mu$. It requires a certain property of the discriminator $D_{\phi}$ to ensure the match of $\nu_{\theta}$ and $\mu$. Discriminators with weak discriminative power, say a constant function, usually drag down the training of the generator due to the incorrect or useless information provided. A candidate discriminator should possess the ability to differentiate between the real distribution and fake distribution in a certain sense, ensuring that the minimized MSE in \eqref{eq:ob} is meaningful. In the next subsection, we will explore the property of discriminator such that minimizing \eqref{eq:ob} can lead to $\nu_\theta=\mu$.

%The property of the discriminator is closely related to the discriminative loss function. The discriminative loss function is usually defined as the objective function of certain divergences, such as JS divergence in GAN \cite{goodfellow2014generative}, and computing such divergence involves finding the optimal discriminator that optimizes the objective function. See \cite{liu2017approximation} for a comprehensive description of the discriminative loss function. However, an optimal discriminator might not be achievable in practice. Hence, we expect a relatively weaker requirement on discriminator to achieve the optimality of MCGAN, and we name this property as \emph{discriminability}.

\begin{definition}[Discriminability]\label{strict_dis_ability}
 A discriminator
 %\begin{equation}
%\begin{aligned}
$$\mathcal{P}(\mathcal{X}) \times \mathcal{P}(\mathcal{X}) \times \mathcal{X}\rightarrow \mathbb{R}; \quad
(\mu,\nu,x) \mapsto D^{\phi_{\mu,\nu}} (x),$$
%\end{aligned}
%\end{equation}
where $\phi_{\cdot,\cdot}:\mathcal{P}(\mathcal{X}) \times \mathcal{P}(\mathcal{X}) \rightarrow \Phi$, is said to have discriminability if there exist two constants $a\in\{-1,1\}$ and $c\in \mathbb{R}$ such that for any two measures $\mu,\nu\in\mathcal{P(X)},$  it satisfies that
\begin{equation}\label{eq:strictdisc}a(D^{\phi_{\mu,\nu}}(x)-c)(p_{\mu}(x)-p_\nu(x))>0,
\end{equation}
for all $x\in\mathcal{A}^{\mu,\nu}:=\{x\in\mathcal{X}:p_{\mu}(x)\neq p_\nu(x)\}.$  We denote the set of discriminators with discriminability as $\mathcal{D}_{\text{Dis}}.$%\supp{p_{\mu}}\cup\supp{p_\nu}
\end{definition}

The discriminability of the discriminator can be interpreted as the ability to distinguish between $\nu$ and $\mu$ point-wisely over $\mathcal{A}^{\mu,\nu}$ by telling the sign (or the opposite sign) of $p_{\mu}(x)-p_{\nu}(x)$. In \eqref{eq:strictdisc}, if $a=1$, the constant $c$ can be regarded as a criterion in the sense that $D^{\phi_{\mu,\nu}}(x)-c$ is positive when $p_{\mu}(x)>p_{\nu}(x)$ and vice versa.

\begin{table}[h]
    \begin{center}
    \resizebox{\linewidth}{!}{%
    \begin{tabular}{ccccc}
    \toprule
    Name & Discriminative loss & $D^*(x)$  & $a$ & $c$\\
    \midrule
    Vanilla GAN  & Binary cross-entropy &  $\frac{p_\mu(x)}{p_\mu(x)+p_{\nu_{\theta}}(x)}$  & $1$ & $1/2$\\
    LSGAN  & Least square loss &  $\frac{\alpha p_\mu(x)+\beta p_{\nu_{ \theta}}(x)}{p_\mu(x)+p_{\nu_{ \theta}}(x)}$ & {$\sign{(\alpha-\beta)}$} & $\frac{\alpha+\beta}{2}$\\
    Hinge GAN & Hinge loss & $2\mathds{1}_{\{p_\mu(x)\geq p_{\nu_{\theta}}(x)\}}-1$  & $1$ & $0$\\
    Energy-based GAN & Energy-based loss & $m\mathds{1}_{\{p_\mu(x)<p_{\nu_{\theta}}(x)\}}$  & {$\sign{(-m)}$} & $\frac{m}{2}$\\
    $f$-GAN & VLB on $f$-divergence & $f'\left(\frac{p_\mu(x)}{p_{\nu_{\theta}}(x)}\right)$  & $1$ & $f'(1)$\\
    \bottomrule
    \end{tabular}}
    \caption{ List of common discriminative loss functions that satisfy strict discriminability}
    \label{tab:discriminability}
    \end{center}
\end{table}

The discriminability covers a variety of optimal discriminators in GAN variants. We present in Table \ref{tab:discriminability} a list of optimal discriminators of some commonly used GAN variants along with their values of $a$ and $c$. The detailed description can be found in Appendix \ref{appendix:discriminability}. Although discriminability can be obtained by training the discriminator via certain $\mathcal{L}_{D}$, it is worth emphasizing that the discriminator does not necessarily need to reach its optimum to obtain discriminability.

\begin{assumption}\label{assumption:H}
Let $H$ be defined in Eqn. \eqref{eqn_H}. The equality $H(\theta,\phi)=\vec{0}$ holds only if $(\theta,\phi)$ reaches the equilibrium point where $\nu_\theta=\mu$.
\end{assumption}
\begin{comment}
This assumption is made because when $H(\theta,\phi)=\vec{0}$ outside the equilibrium point,  it typically indicates a failed discriminator (say a constant function), which is of no interest to us. By making this assumption, we focus on scenarios where the discriminator is sufficiently powerful to guide the generator training.
\end{comment}
Now, we establish the optimality of $\mu=\nu_\theta$ in the following theorem under the regularity condition (Assumption \ref{assumption:H}). 
\begin{restatable}{thm}{theoremg}\label{theorem:g}
Assume Assumption \ref{assumption:H} holds, and let  $\phi'_{\cdot,\cdot}:\mathcal{P(X)}\times\mathcal{P(X)}\rightarrow\Phi$ be a parameterization map such that $D^{\phi'_{\cdot,\cdot}}:\mathcal{P}(\mathcal{X}) \times \mathcal{P}(\mathcal{X}) \times \mathcal{X}\rightarrow \mathbb{R}$ has discriminability, i.e. $D^{\phi'_{\cdot,\cdot}}\in \mathcal{D}_{\text{Dis}}$. If $\theta^*$ is a local minimizer of $\mathcal{L}_{G}(\theta;\phi'_{\mu,\nu_{\theta}},\mu)$ defined in \eqref{eq:unobj}, then $\nu_{\theta^*} =\mu$.
\end{restatable}

Theorem \ref{theorem:g} implies that MCGAN can effectively learn the data distribution $\mu$ without requiring the discriminator to reach its optimum; the discriminability is sufficient, which is again attributed to the strong supervision provided by regression loss $\mathcal{L}_{R}$. We defer the proof of Theorem \ref{theorem:g} and other theoretical properties of MCGAN, e.g., improved training stability and relation to $f$-divergence to the Appendix.

\section{Numerical experiments}\label{sec:experiment}
To validate the efficacy of the proposed MCGAN method, we conduct extensive experiments on a broad range of data, including image, time series, and video data for various generative tasks. For image generation, the conditioning variables are categorical, whereas for time series and video generation tasks, the conditioning variables are continuous. To show the flexibility of MCGAN to enhance different GAN backbones, we choose several state-of-the-art GAN models with different discriminative losses (i.e., BCE and Hinge loss) as baselines. Various test metrics and qualitative analysis are employed to give a comprehensive assessment of the quality of synthetic data generation.  
 
The full implementation details of numerical
experiments, including models, test metrics, hyperparameters, optimizer and supplementary numerical results, can be found in Appendix \ref{appendix:numdetail}. Moreover, we will open-source the codes and final checkpoints upon publication for reproducibility.

\subsection{Unconditional and conditional image generation}
\subsubsection{Datasets}
We conduct conditional image generation tasks using the CIFAR-10 and CIFAR-100 datasets \cite{cifardataset}, which are standard benchmarks with 60K 32x32 RGB images across 10 and 100 classes, respectively. To further validate our MC method on larger and higher-resolution datasets, we include: 1) the unconditional FFHQ256 dataset, which contains 70K 256x256 human face images, 2) the conditional ImageNet64 dataset, which has 1.2 million 64x64 images across 1,000 classes, and 3) the unconditional LSUN bedroom data, which has 3 million 256x256 images.

\begin{comment}We conduct conditional image generation tasks on the CIFAR-10 and CIFAR-100 datasets \cite{cifardataset}.  To further validate our MC method on large-scale and high-resolution datasets, we also include the unconditional FFHQ256 (high-resolution) dataset and the conditional lmageNet64 (large-scale) dataset.\textbf{CIFAR-10} is a widely used benchmark dataset in
conditional image generation tasks. It consists of 60,000 $32\text{x}32$ RGB
images for 10 different classes with 6,000 images each. \textbf{CIFAR-100} is similar to CIFAR-10 in terms of image size and dataset size but has 100 classes containing 600 images each. \textbf{FFHQ256} dataset is a subset of Flickr-Faces-HQ (FFHQ) dataset, which consists of 70k 256x256 human face images and is commonly used to test the generative moldel's capability of generating realistic images. \textbf{ImageNet64} has 1.2M 64x64 of 1000 different classes, which is a challenging benchmark for the conditional image generation task due to its diversity.
\end{comment}

We validate our method using two different backbones, BigGAN
\cite{brock2018large} and StyleGAN2 \cite{karras2020analyzing}. The test metrics include
\emph{Inception Score} (IS), \emph{Fr\'echet Inception Distance} (FID), and
\emph{Intra Fr\'echet Inception Distance} (IFID) together with two recognizability
metrics \emph{Weak Accuracy} (WA) and \emph{Strong Accuracy} (SA). To alleviate the overfitting and improve the generalization,  we also increase data efficiency by using the
\emph{Differentiable Augmentation} (DiffAug) \cite{zhao2020differentiable}.

In the following, we mainly focus on the CIFAR-10 dataset for in-depth analysis, with a brief summary of the results on the other datasets.

%\subsubsection{Learning curves in training}
\subsubsection{Faster training convergence}
In Figure \ref{fig:CIFAR-10_fidis_BigGAN}, we plot the learning curves in terms of FID and IS during the training. It shows that the MC method tends to have much faster convergence and ends at a considerably better level in both baselines of using Hinge loss and BCE loss. 

%\subsubsection{Quantitative results}
\subsubsection{Improved fidelity metrics}
As shown in Table \ref{table_CIFAR-10}, our MC method considerably improves all the baselines independently of the choice of discriminative loss ($\mathcal{L}_{D}$). Specifically, when using Hinge loss as $\mathcal{L}_{D}$ along with DiffAug, the MC method improves the FID from 4.43 to 3.61, comparable to the state-of-the-art FID result of \cite{kang2022studiogan}. Also, its IS score is significantly increased from 9.61 to 9.96, indicating better diversity of the generated samples.

In addition, applying the MC method to the cStyleGAN2 backbone results in an FID improvement of approximately 0.08. Notably, the combination of Hinge + MC + DiffAug achieves an FID of 2.16, which, to our knowledge, is the best FID achieved using StyleGAN2 as the backbone \cite{kang2021rebooting,kang2022studiogan,tseng2021regularizing} 

\newcolumntype{Y}{>{\centering\arraybackslash}X}
\begin{table}[!htbp]
\centering
\resizebox{\columnwidth}{!}{%
\begin{tabularx}{0.9\textwidth}{l| Y|Y| Y||Y|Y|Y}
\toprule
    \multicolumn{1}{l|}{Loss} & \multicolumn{3}{c||}{Hinge} &   \multicolumn{3}{c}{BCE}   
\\\midrule
\multicolumn{1}{l|}{Metrics} & \multicolumn{1}{c|}{IS $\uparrow$} &   \multicolumn{1}{c|}{FID $\downarrow$} &
\multicolumn{1}{c||}{IFID $\downarrow$}&
\multicolumn{1}{c|}{IS $\uparrow$} &   \multicolumn{1}{c|}{FID $\downarrow$} & \multicolumn{1}{c}{IFID $\downarrow$}
\\\midrule
BigGAN  & 9.27 \textcolor{gray}{\scriptsize $\pm$ 0.11}   & 5.31 & 16.20 & 9.30 \textcolor{gray}{\scriptsize $\pm$ 0.14}  & 5.55  & 16.62
\\\midrule
 \quad+DiffAug   & 9.61 \textcolor{gray}{\scriptsize $\pm$ 0.06} &  4.43 & 14.60 & 9.51 \textcolor{gray}{\scriptsize $\pm$ 0.11}  &  4.71  & 14.83
\\
 \quad+MC   & 9.66 \textcolor{gray}{\scriptsize $\pm$ 0.09}  &  4.51 & 14.71 & 9.62 \textcolor{gray}{\scriptsize $\pm$ 0.09}     &  4.61 &  14.82 
\\
 \quad+MC+DiffAug  & \textbf{9.96} \textcolor{gray}{\scriptsize $\pm$ 0.12} &  \textbf{3.61} & \textbf{13.60}& \textbf{9.94} \textcolor{gray}{\scriptsize $\pm$ 0.10} &  \textbf{3.93} & \textbf{13.72}
 \\\midrule
 StyleGAN2 & -& -& - & 10.17 \textcolor{gray}{\scriptsize $\pm$ 0.12} &  3.7& 14.04\\\midrule
 \quad+DiffAug   & 10.19 \textcolor{gray}{\scriptsize $\pm$ 0.11} &  2.25& 11.40& 10.03 \textcolor{gray}{\scriptsize $\pm$ 0.09}&  2.44& 11.62\\
\quad+MC+DiffAug  & \textbf{10.26} \textcolor{gray}{\scriptsize $\pm$ 0.08} & \textbf{2.16} & \textbf{11.04} & \textbf{10.10} \textcolor{gray}{\scriptsize $\pm$ 0.11} &\textbf{2.36} & \textbf{11.30}
% \\\midrule
% LeCAM(StyleGAN2+ADA)  &  -  & - & -& 10.53  & 2.31 &  11.82 
 \\\bottomrule
\end{tabularx}}
\caption{Quantitative results of image generation on CIFAR-10 using BigGAN/StyleGAN2 w/o and with our MC method and Differentiable Augmentation.}\label{table_CIFAR-10}
\end{table}

\begin{figure}[ht]
 \includegraphics[width=0.47\textwidth]{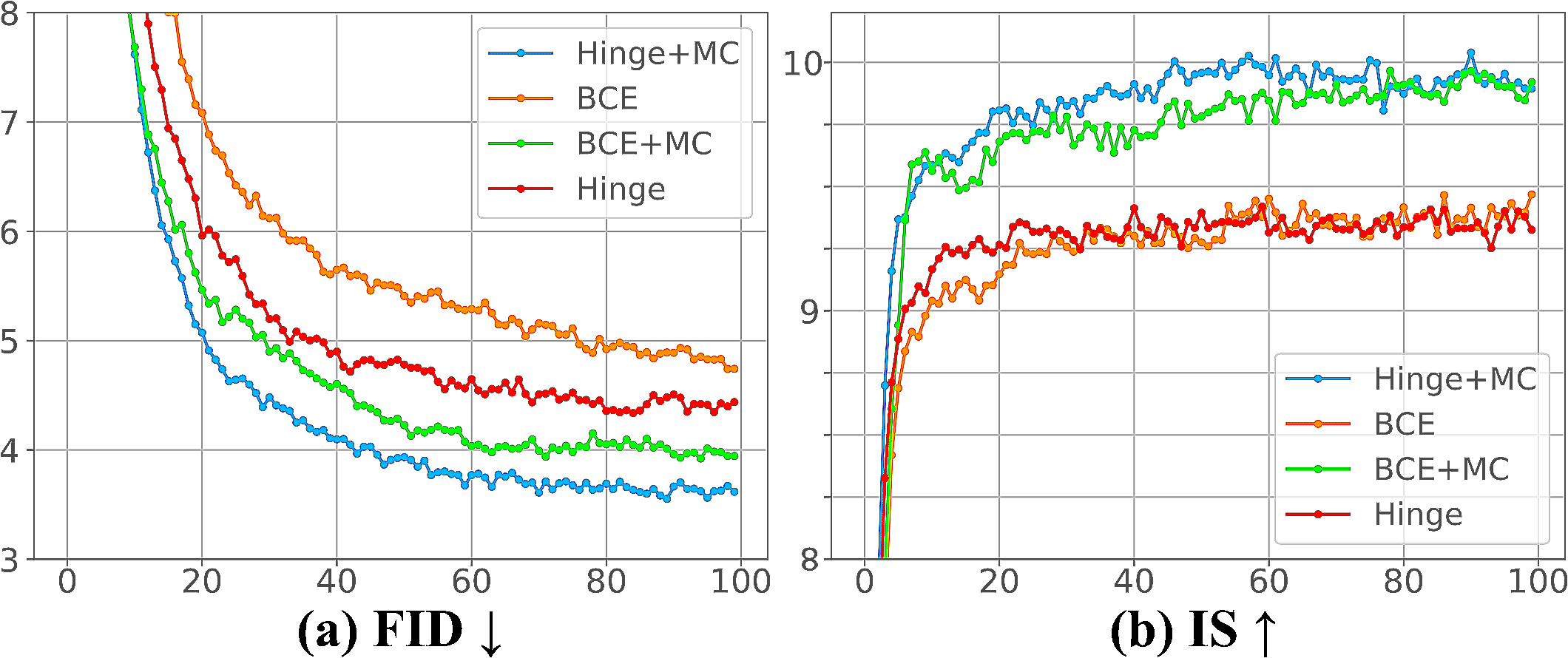}
 \caption{The learning curves in terms of (a) Fr\'echet Inception Distance and (b) Inception Score along the training on the CIFAR-10 dataset using BigGAN with different loss combinations.}\label{fig:CIFAR-10_fidis_BigGAN}
\end{figure}

\subsubsection{Improved recognizability metrics} We generated 10k (the same setting as the test set) images using the BigGAN backbone. The WA rates are 62.56\%, 52.09\%, and 54.71\% for
the real test set, the generated set from Hinge baseline, and the generated set from Hinge + MC, respectively. Our MC method's images perform closer to the
real test set than the baseline's, showing better distribution matching to the real data in terms of recognizability. The SA rate of our MC method is
83.42\% compared to 93.65\% of the real test set, showing that we generate fairly
recognizable fake images.

\subsubsection{Qualitative results}
The qualitative results are shown in Figure \ref{fig:cifar10_biggan} and Figure \ref{fig:stylegan2} in Appendix with only a small amount of images (in red boxes) misclassified by our classifier.

\begin{figure}[ht]
    \centering
\includegraphics[width=0.8\linewidth]{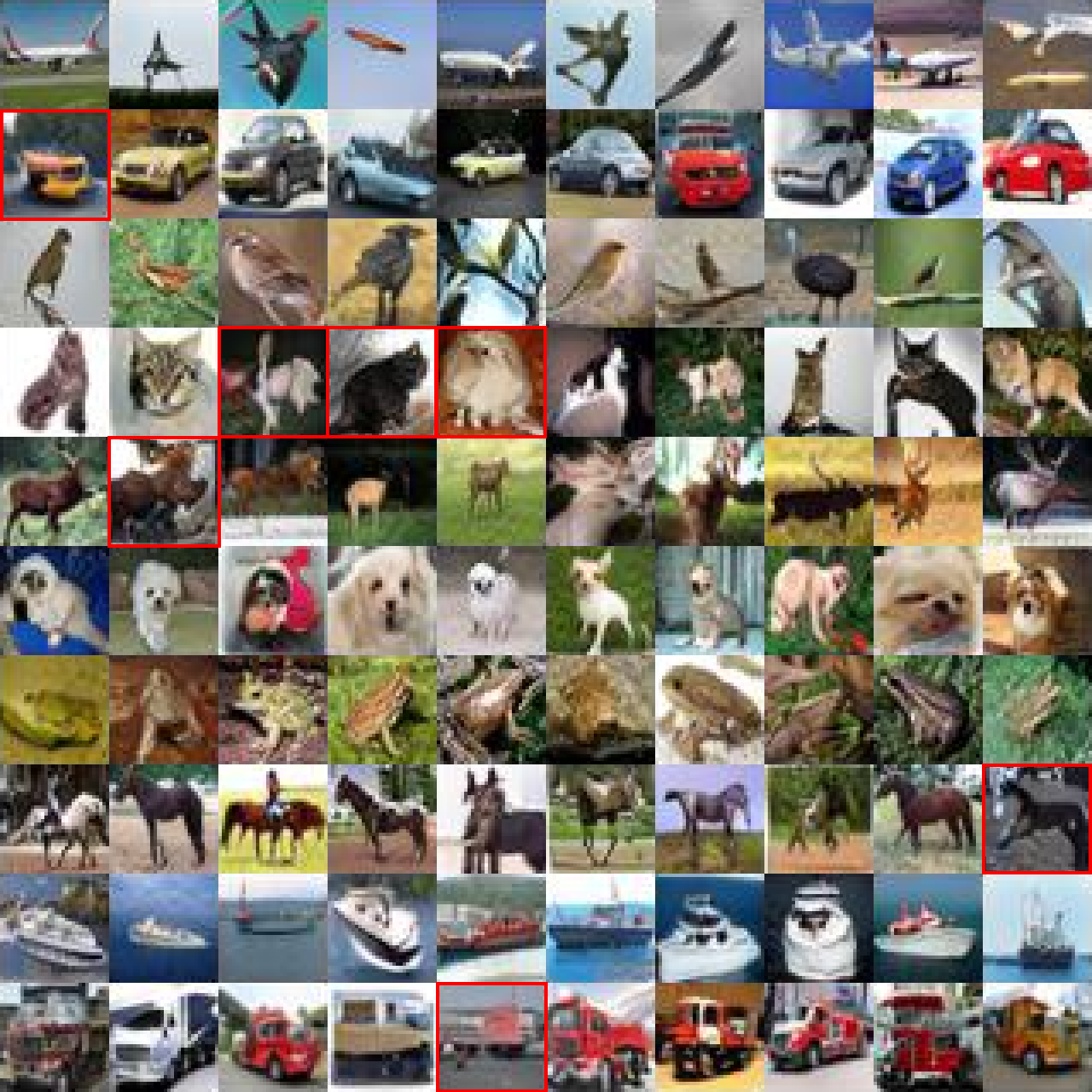}
    \caption{CIFAR-10 samples generated by the BigGAN backbone trained via Hinge + DiffAug + MC. Images in each row belong to one of the 10 classes. Images misclassified by ResNet-50 are in red boxes.}
    \label{fig:cifar10_biggan}
\end{figure}
%\input{Figures/CIFAR10/CIFAR10_StyleGAN2}
%\newcolumntype{Y}{>{\centering\arraybackslash}X}
\begin{table}[ht]
\centering
\resizebox{\columnwidth}{!}{%
\begin{tabularx}{0.95\textwidth}{l| Y | Y| Y||Y|Y|Y}
\toprule
    \multicolumn{1}{l|}{Loss} & \multicolumn{3}{c||}{Hinge} &   \multicolumn{3}{c}{BCE}   
\\\midrule
\multicolumn{1}{l|}{Metrics} & \multicolumn{1}{c|}{IS $\uparrow$} &   \multicolumn{1}{c|}{FID $\downarrow$} &
\multicolumn{1}{c||}{IFID $\downarrow$}&
\multicolumn{1}{c|}{IS $\uparrow$} &   \multicolumn{1}{c|}{FID $\downarrow$} & \multicolumn{1}{c}{IFID $\downarrow$}
\\\midrule
BigGAN  & 10.73 \textcolor{gray}{\scriptsize$\pm$0.10} & 8.31 & 83.36 &  10.81 \textcolor{gray}{\scriptsize$\pm$0.14}   & 8.37 & 81.89
\\\midrule
\quad+DiffAug   & 10.72 \textcolor{gray}{\scriptsize$\pm$0.13}  &  7.37 & 80.00 &  10.71 \textcolor{gray}{\scriptsize$\pm$0.08}   & 7.61 & 80.48 
\\
\quad+MC   & 11.39   \textcolor{gray}{\scriptsize $\pm$0.10}  &  6.97 & 80.20 &   11.59 \textcolor{gray}{\scriptsize$\pm$0.12 }  &  6.99 & 80.91
\\
\quad+MC+DiffAug  & \textbf{11.81} \textcolor{gray}{\scriptsize $\pm$ 0.06} & \textbf{5.77} & \textbf{76.26} & \textbf{11.90} \textcolor{gray}{\scriptsize $\pm$0.08} &  \textbf{5.85} & \textbf{77.33}
\\
 \bottomrule
\end{tabularx}}
\caption{Quantitative results of image generation on CIFAR-100 using BigGAN w/o and with our MC method and Differentiable Augmentation.}\label{table_CIFAR100_loss}
\end{table} 

\subsubsection{Latent space analysis} The latent space learned by the generator
is expected to be continuous and smooth so that small perturbations on the
conditional input can lead to smooth and meaningful modifications on the
generated output. To explore the latent space, we interpolate between each pair
of randomly generated images by linearly interpolating
their conditional inputs. The results are shown in Figure \ref{fig:interp_two}.
Intermediary images between a pair of images from two different classes are
shown in each row with their confidence score distributions below. The labels of the two classes are shown on the left and right sides of each row, respectively. Each distribution of
the confidence scores is calculated by the bottleneck representation of the
ResNet-50 classifier with a softened softmax function of temperature 5.0 for normalization. The
score bars of the left class and the right class are shown in green and magenta,
respectively. The red boxes highlight the images being classified as a third class, while the yellow boxes highlight the images having non-monotonic transitions of their confidence scores
compared to those of their adjacent images. In other words, images in both red
and yellow boxes are undesirable as they imply that the latent space is less
continuous and less smooth. By comparing Figure \ref{fig:interp_two}a and \ref{fig:interp_two}b, we
can see that the MC method outperforms in the learned latent space and has most of the decision switch between the two classes occur in the middle range of the interpolation.

\begin{figure}[ ht]
    \centering
\includegraphics[width=1\linewidth]{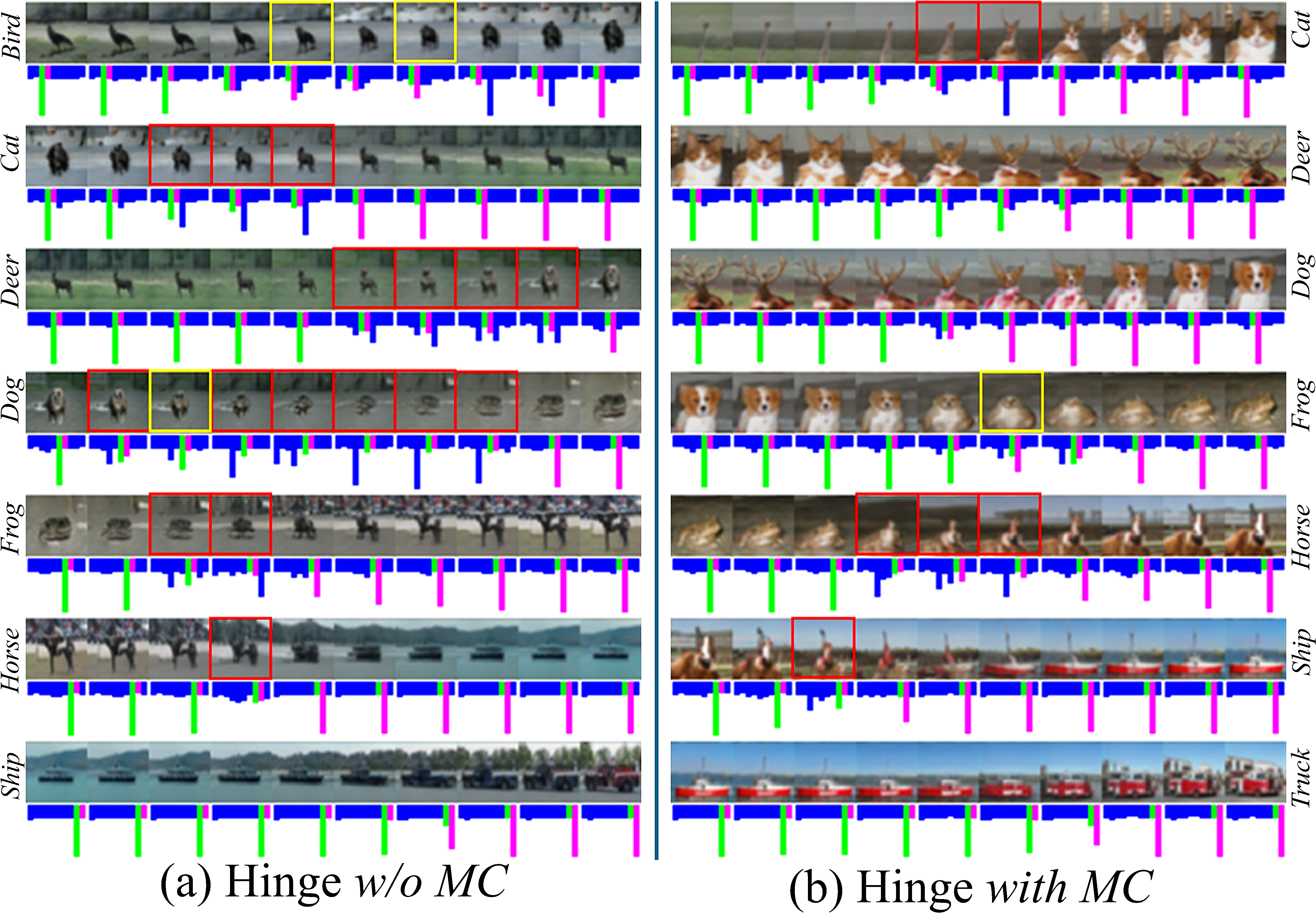}
    \caption{Latent space interpolation based on cStyleGAN2 backbone trained via Hinge loss w/o and with our MC method. Red and yellow boxes highlight two types of undesirable transitions between generated images.}
    \label{fig:interp_two}
\end{figure}

%\subsection{to-do results}
%\noteNi{some results to show training stability. it is originally in sec 3. should be incorporated into sec 5.}\textcolor{blue}{BR: added, start with 'In Figure 5, we also plot the...'.}

\subsubsection{Quantitative results on CIFAR-100}
For completeness, we show the image generation performance on CIFAR-100 in Table \ref{table_CIFAR100_loss}. Significant improvements are achieved by using our MC
method independently for both baseline discriminative losses, with an average improvement of 1.1 in IS, 1.6 in FID, and 3.7 in IFID. A detailed sensitive analysis w.r.t the Monte Carlo sample size is provided in the appendix.
\subsubsection{Large-scale and high-resolution dataset results}
For the FFHQ256 (high-resolution), the lmageNet64 (large-scale), and the LSUN bedroom (large-scale high-resolution) dataset, we use the StyleGAN2-ada \cite{karras2020training} as backbones. As shown in Table \ref{table_largescale}\footnote{Baseline results differ from StyleGAN2-ADA's official benchmarks due to hyperparameter adjustments for different GPU setups.}, MCGAN achieved significant and consistent gains in both FID and IS, as evidenced by 16.4\% ($4.51\rightarrow3.77$), 15.5\% ($19.83\rightarrow16.76$), and 35.7\% ($4.34\rightarrow2.79$) FID improvement, respectively, on FFHQ256, ImageNet64, and LSUN bedroom datasets. These improvements are significant and consistent during training periods and across various datasets, demonstrating faster convergence and better generation ability. 
\begin{table}[ht]
\centering
\resizebox{\columnwidth}{!}{%
\begin{tabularx}{0.9\textwidth}{c| l| Y | Y|Y|Y}
\toprule
 \multicolumn{1}{c|}{Dataset}&\multicolumn{1}{l|}{Method} &   \multicolumn{1}{c|}{FID $\downarrow$} &
\multicolumn{1}{c|}{IS $\uparrow$}&
\multicolumn{1}{c|}{Precision$\uparrow$} &
\multicolumn{1}{c}{Recall $\uparrow$}
\\\midrule
 \multirow{2}{*}{FFHQ256} & original  & 4.51 \textcolor{gray}{\scriptsize $\pm$ 0.03}& 5.10 \textcolor{gray}{\scriptsize $\pm$ 0.07} & \textbf{0.69} & 0.40 \\
&\quad +MC   & \textbf{3.77} \textcolor{gray}{\scriptsize $\pm$ 0.04} & \textbf{5.25} \textcolor{gray}{\scriptsize $\pm$ 0.06} & 0.69 & \textbf{0.45}
 \\\midrule
%  \multirow{2}{*}{AFHQ512} & original  & 7.04 & 10.90 & 0.74 & \textbf{0.36} \\
%&\quad +MC   & \textbf{6.52} & \textbf{11.33} & \textbf{0.78} & 0.29
 %\\\midrule
 \multirow{2}{*}{ImageNet64} &original  & 19.83 \textcolor{gray}{\scriptsize $\pm$ 0.02} & 13.67 \textcolor{gray}{\scriptsize $\pm$ 0.17} &  \textbf{0.65}  & 0.33
\\
&\quad +MC  & \textbf{16.76} \textcolor{gray}{\scriptsize $\pm$ 0.08} &  \textbf{13.96} \textcolor{gray}{\scriptsize $\pm$ 0.22} & 0.63 &  \textbf{0.43}
 \\ \midrule
\multirow{2}{*}{LSUN bedroom} &original  & 4.34 \textcolor{gray}{\scriptsize $\pm$ 0.03} & 2.45 \textcolor{gray}{\scriptsize $\pm$ 0.02} & 0.57   & 0.22
\\
&\quad +MC  & \textbf{2.79} \textcolor{gray}{\scriptsize $\pm$ 0.01} & \textbf{2.45}  \textcolor{gray}{\scriptsize $\pm$ 0.02} & \textbf{0.61} & \textbf{0.23} 
 \\
\bottomrule
\end{tabularx}}
\caption{Quantitative results of image generation on large-scale and high-resolution datasets using StyleGAN2-ada w/o and with our MC method; FID is 10-run average.}\label{table_largescale}
\end{table}

\subsection{Conditional video generation}
%Inspired by the effectiveness of MCGAN in generation tasks on spatial image data, we also validate our MCGAN model in conditional video generation task. In this task, the generator aims to generate future video frame given the past frames. 
The conditional video generation task aims to generate the next frame given the past frames of the videos. Here, we used the Moving MNIST data set \cite{srivastava2015unsupervised}, which consists of 10,000 20-frame 64x64 videos of moving digits. The whole dataset is divided into the training set (9,000 samples) and the test set (1,000 samples). For the architecture of both the generator and discriminator, we use the convolutional LSTM (ConvLSTM) unit proposed by \cite{shi2015convolutional} due to its effectiveness in video prediction tasks.  In the model training, the generator takes in 5 past frames as the input and generates the corresponding 1-step future frame, then the real past frames and the generated future frames are concatenated along time dimension and put into the discriminator. 

For comparison, we used classical GAN as the benchmark. We trained our model for 20,000 epochs with batch size 16. The model performance is evaluated by computing the MSE between the generated frames and the corresponding ground truth on the test set. Numerical results show that our proposed MC method reduces GAN's MSE from 0.1012 to 0.0840. Compared to the baseline, the predicted frames from our MC method are clearer, more coherent, and visually closer to the ground truth, as shown in Figure \ref{fig:video}.

\begin{figure}[ht]
    \centering
\includegraphics[width=0.7\linewidth]{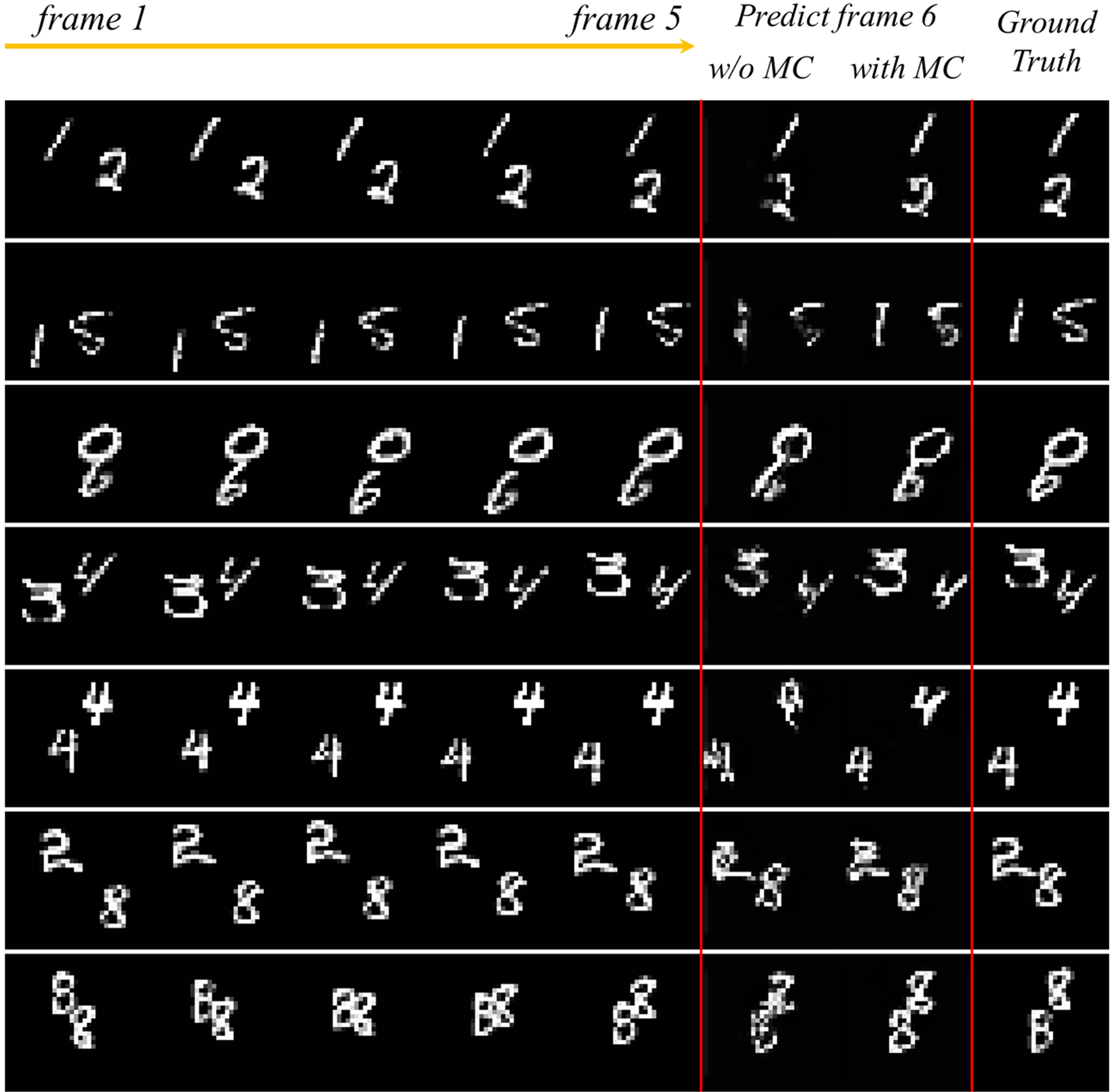}
    \caption{Results of predicting the next frame given the past 5 frames using ConvLSTM w/o and with our MC method.}
    \label{fig:video}
\end{figure}

\subsection{Conditional time-series generation}
%({\color{red} \cite{yoon2019time,esteban2017real,li2015generative} are refs correct?}
Following \cite{liao2024sig}, we consider the conditional time-series generation task on two types of datasets (1) $d$-dimensional vector auto-regressive (VAR) data and (2) empirical stock data. The goal is to generate 3-step future paths based on the 3-lagged value of time series. We apply the MCGAN to the RCGAN baseline \cite{esteban2017real} and bencharmak it with  TimeGAN \cite{yoon2019time}, GMMN \cite{li2015generative} and SigWGAN \cite{liao2024sig} as the strong SOTA models for conditional time series generation. The model performance is evaluated using metrics in \cite{liao2024sig} including (1) ABS metric, (2) Correlation metric, (3)ACF metric and (4) $R^2$ error 
to assess the fitting of synthetic data in terms of marginal distribution, correlation, autocorrelation and usefulness, respectively.

%(1) ABS metric - a histogram-based metrics that measure the distance between two distributions; (2) Correlation metric - a metric that computes the difference between two correlation matrices; (3) ACF metric - a metric that measures the difference in temporal dependency; (4) $R^2$ error - a metric that measures the usefulness of generated samples. 
 
%IN  which showcases its effectiveness in synthesising streamed data and superiority in capturing the temporal dependency.
\subsubsection{VAR dataset}
To validate the performance of MCGAN for multivariate time series systematically, we apply our method on VAR datasets with various path dimensions $d \in [1, 100]$ and different parameter settings. For $d \in \{1, 2, 3\}$, MCGAN consistently outperforms the RCGAN and TimeGAN (see Table \ref{table_Var_dim=1} -  \ref{table_Var_dim=3} in the Appendix). Figure \ref{fig:VARConDist} shows that the MCGAN and SigCWGAN have a better fitting than other baselines in terms of conditional law as the estimated mean (and standard deviation) is closer to that of the ground truth compared with the other baselines for $d=3$.  Note that SigCWGAN suffers the curse of dimensionality resulting from large $d$ and becomes infeasible for $d\geq 50$, whereas MCGAN does not. In fact, as $d$ increases, the performance gains of MCGAN become more pronounced, which is shown in Table \ref{table_VAR_loss}. For example, with $d = 100$, the MC method improves all the metrics by 30\%-40\%, further highlighting its effectiveness in high-dimensional settings.
\begin{table}[ht]
\centering
\resizebox{\columnwidth}{!}{%
\begin{tabularx}{0.8\textwidth}{c|l| Y | Y| Y||Y|Y|Y}
\toprule
    \multicolumn{2}{l|}{Loss} & \multicolumn{3}{c||}{Hinge} &   \multicolumn{3}{c}{BCE}   
\\\midrule
\multicolumn{1}{c|}{$d$} & \multicolumn{1}{l|}{Method} & \multicolumn{1}{c|}{ABS $\downarrow$} &   \multicolumn{1}{c|}{Corr $\downarrow$} &
\multicolumn{1}{c||}{ACF $\downarrow$}&
\multicolumn{1}{c|}{ABS $\downarrow$}&  \multicolumn{1}{c}{Corr $\downarrow$} &
\multicolumn{1}{c}{ACF $\downarrow$}
\\\midrule
\multirow{2}{*}{10} & RCGAN & 0.01802 & 0.05678 & 0.08175 & 0.01525 & 0.05074 & 0.08995 \\
 & \quad +MC & \textbf{0.01550} & \textbf{0.04361} & \textbf{0.06511} & \textbf{0.01393} & \textbf{0.04589} & \textbf{0.07193} \\\midrule
\multirow{2}{*}{50} & RCGAN & 0.03527 & 0.07000 & 0.08844 & 0.03632 & 0.07103 & 0.08765 \\
 & \quad +MC & \textbf{0.02861} & \textbf{0.06161} & \textbf{0.06869} & \textbf{0.02503} & \textbf{0.05995} & \textbf{0.06996} \\\midrule
\multirow{2}{*}{100} & RCGAN & 0.03319 & 0.07895 & 0.11018 & 0.03788 & 0.07295 & 0.10222 \\
 & \quad +MC & \textbf{0.02296} & \textbf{0.04981} & \textbf{0.06687} & \textbf{0.02344} & \textbf{0.05024} & \textbf{0.06143} 
 \\\bottomrule
\end{tabularx}}
\caption{Quantitative results of time-series generation on VAR data with different path dimensions $d$ ranging from $10$ to $100$ using RCGAN w/o and with our MC method.}\label{table_VAR_loss}
\end{table}

%\newcolumntype{Y}{>{\centering\arraybackslash}X}

\begin{figure}[htb]
\centering % <-- added
\hspace*{0cm}\includegraphics[width=1\linewidth]{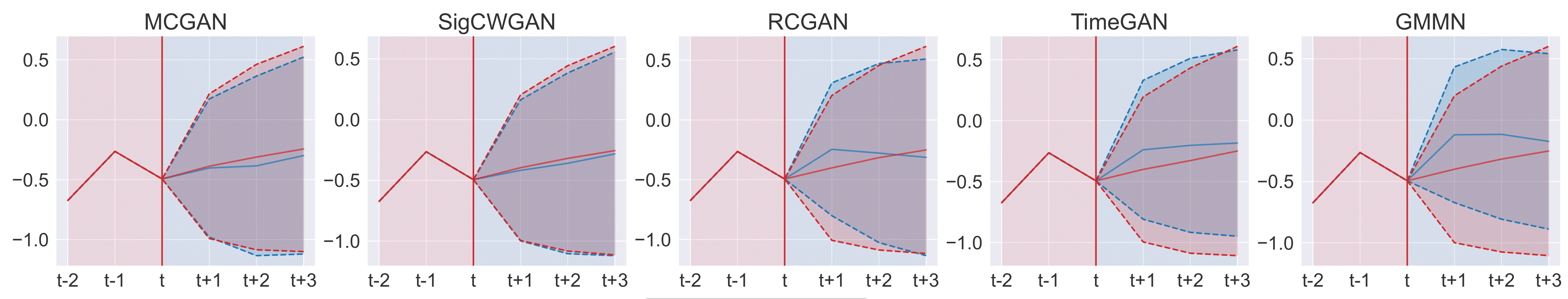}
\caption{Comparison of all models' performance in fitting the conditional distribution of future time series given one past path sample. The real and generated paths are plotted in red and blue, respectively, with the shaded area as the $95\%$ confidence interval. The training dataset is synthesized from VAR(1) model for $d=3$, $\phi = 0.8$ and $\sigma = 0.8$.}
\label{fig:VARConDist}
\end{figure}

\subsubsection{Stock dataset}
The stock dataset is a 4-dimensional time series composed of the log return and log volatility data of S$\&$P 500 and DJI spanning from 2005/01/01 to 2020/01/01. To cover the stylized facts of financial time series like leverage effect and volatility clustering, we also evaluate our generated samples using the ACF metric on the absolute return and squared return. Table \ref{table_stock_loss} demonstrates that our MC method consistently improves the generator performance in terms of temporal dependency, cross-correlation and usefulness. Although  RCGAN achieved comparable ABS metrics, it failed to capture the cross-correlation and temporal dependence. Specifically, using our proposed MC method, the correlation metric and ACF metric of RCGAN  can be improved from 0.25184 to 0.15687 and from 0.03814 to 0.02905. The gap in the $R^2$ further showcases that our MC method can enhance the generator to generate high-fidelity samples.

\begin{table}[ ht]
\centering
\resizebox{\columnwidth}{!}{%
\begin{tabularx}{0.8\textwidth}{c| Y | Y| Y| Y| Y| Y}
\toprule
 \multicolumn{1}{c|}{Model}& \multicolumn{1}{c|}{ABS $\downarrow$} &
\multicolumn{1}{c|}{ACF $\downarrow$}&
\multicolumn{1}{c|}{ACF($|x|$) $\downarrow$}&
\multicolumn{1}{c|}{ACF($x^2$) $\downarrow$}&
\multicolumn{1}{c|}{Corr $\downarrow$}&
\multicolumn{1}{c}{$R^2$ (\%) $\downarrow$} 
\\\midrule
RCGAN&\textbf{0.00868}&0.03814 &0.07874&0.13933&	0.25184&	4.49683\\
MCGAN (ours)&	0.00996	&\textbf{0.02905}&\textbf{0.05437}&0.09933&	0.15687&\textbf{2.84285}
\\\midrule
SigCWGAN &0.00960 &0.02982&0.13385&\textbf{0.08456}&	\textbf{0.11721}&	3.81981\\
GMMN&	0.01389&	0.05989&0.25295&0.26960&	0.31838&	11.87578\\
TimeGAN&	0.01100&	0.05716&0.06899&0.12584&	0.47344&	4.53960
 \\\bottomrule
\end{tabularx}}
\caption{Quantitative results of time-series generation on SPX/DJI data using RNN w/o and with our MC method.}\label{table_stock_loss}
\end{table}

\section{Conclusion}
This paper presents a general MCGAN method to tackle the training instability, a key bottleneck of GANs. Our method enhances generator training by introducing a novel regression loss for (conditional) GANs.  We establish the optimality and discriminability of MCGAN, and prove that the convergence of optimal generator can be achieved under a weaker condition of the discriminator due to the strong supervision of the regression loss. Moreover, extensive numerical results on various datasets, including image, time series data, and video data, are provided to validate the effectiveness and flexibility of our proposed MCGAN and consistent improvements over the benchmarking GAN models.  

For future work, it is worthwhile to explore the application of MCGAN to enhance state-of-the-art GAN models for more challenging and complex tasks, such as text-to-image generation. Moreover, given the flexibility and promising results of the MCGAN on different types of data, it can be effectively applied to generate multi-modality datasets simultaneously.

%The flexibility of our MCGAN framework allows us to integrate cutting-edge generator architectures tailored to specific generation tasks.  

%For example, it is promising to use more sophisticated generator architectures than BigGAN and StyleGAN, e.g., StyleGAN3\cite{karras2021alias}, to further boost model performance for the image generation, while StyleGAN-T \cite{sauer2023stylegan} and GigaGAN \cite{kang2023scaling} are more suitable for large-scale text-to-image generation. Moreover, given the flexibility and promising results of the MCGAN on different types of data, it can be effectively applied to generate multi-modality datasets simultaneously.

%The MCGAN incorporates MSE and Monte Carlo method into the generative loss function that computes the distance between the real and fake expected discriminator output. We analysed that under relaxed requirements on the discriminator, the generator can still learn the target distribution with improved stability. It leverages the strong supervision of MSE and significantly enhances the performance of BigGAN on the conditional image generation task using CIFAR-10 and CIFAR-100 datasets. 

\section{Acknowledgements}
HN and WY are supported by the EPSRC under the program grant EP/S026347/1 and the Alan Turing Institute under the EPSRC grant EP/N510129/1. WY is also supported by the Data Centric Engineering Programme (under the Lloyd's Register Foundation, UK grant G0095). The authors are grateful to Richard Hoyle from UCL, Terry Lyons from Oxford, and the Oxford Mathematical Institute IT staff for their support with computing resources. We also thank Lei Jiang for his help with some of the numerical experiments and the anonymous referees for their constructive suggestions, which significantly improved the paper. 
%{\color{red} example}
%AAAI is especially grateful to Peter Patel Schneider for his work in implementing the original aaai.sty file, liberally using the ideas of other style hackers, including Barbara Beeton. We also acknowledge with thanks the work of George Ferguson for his guide to using the style and BibTeX files --- which has been incorporated into this document --- and Hans Guesgen, who provided several timely modifications, as well as the many others who have, from time to time, sent in suggestions on improvements to the AAAI style. We are especially grateful to Francisco Cruz, Marc Pujol-Gonzalez, and Mico Loretan for the improvements to the Bib\TeX{} and \LaTeX{} files made in 2020.

%The preparation of the \LaTeX{} and Bib\TeX{} files that implement these instructions was supported by Schlumberger Palo Alto Research, AT\&T Bell Laboratories, Morgan Kaufmann Publishers, The Live Oak Press, LLC, and AAAI Press. Bibliography style changes were added by Sunil Issar. \verb+\+pubnote was added by J. Scott Penberthy. George Ferguson added support for printing the AAAI copyright slug. Additional changes to aaai25.sty and aaai25.bst have been made by Francisco Cruz, Marc Pujol-Gonzalez, and Mico Loretan.
%\newpage
\bibliography{aaai25}

\newpage

\clearpage
%\onecolumngrid
\appendix
\section{Algorithm}\label{appendix:magan_algo}
The pseudocode of the MCGAN is provided in Algorithm \ref{algo:mcgan}.
\begin{algorithm}[!ht]
\caption{Algorithm of MCGAN}\label{algo:mcgan}
\begin{algorithmic}[1]
    \State \textbf{Input}: 
    \Statex \hspace{\algorithmicindent} $N$: number of epochs;
    \Statex \hspace{\algorithmicindent} $N_D$: number of discriminator iterations per generator iteration;
    \Statex \hspace{\algorithmicindent} $B \in \mathbb{N}$: batch size;
    \Statex \hspace{\algorithmicindent} $N_{MC}$: number of Monte Carlo samples;
    \Statex \hspace{\algorithmicindent} $f_1,f_2$: specified functions used to compute discriminative loss;
    %\Statex \hspace{\algorithmicindent} $C_{(lb,ub)}$: clamp function with lower bound $lb$ and upper bound $ub$.
    \State \textbf{Output}: 
    \Statex \hspace{\algorithmicindent} $(\theta^*, \psi^*)$: approximation of the optimal parameters of the generator and discriminator.
    \State Initialize model parameters $(\theta, \psi)$ for generator $G$ and discriminator $D$.
    \For{$n = 1$ \textbf{to} $N$}
        \For{$n_d = 1$ \textbf{to} $N_D$}
            \State Sample batch $\{(x_i, y_i)\}_{i=1}^{B} \sim p_d(X, Y)$
            \State Generate samples $\{(\hat{x}_i, y_i)\}_{i=1}^{B} \sim p_{\theta}(X, Y)$
            \State Compute discriminative loss:
            \[
            \begin{aligned}
            \mathcal{L}_D(\psi; \mu, \nu_{\theta}) \gets &\frac{1}{B} \sum_{i=1}^{B} f_1(D^{\psi}(x_i, y_i)) \\&+ \frac{1}{B} \sum_{i=1}^{B} f_2(D^{\psi}(\hat{x}_i, y_i))
            \end{aligned}
            \]
            \State Update discriminator parameters:
            \[
            \psi \gets \text{Adam}(\mathcal{L}_D)
            \]
        \EndFor
        \State Sample batch $\{(x_i, y_i)\}_{i=1}^{B} \sim p_d(X, Y)$;
        \State For each label $y_i$, estimate the conditional expectation:
        \[
        \hat{\mathbb{E}}_{p_\theta}[D^{\psi}(X, y_i) \mid y_i] \gets
        \frac{1}{N_{MC}} \sum_{j=1}^{N_{MC}} D^{\psi}(G^{\theta}(y_i, z^{(j)}), y_i);
        \]
        %\[
        %\hat{\mathbb{E}}_{p_\theta}[C_{(lb,ub)}(D^{\psi}(X, %y_i)) \mid y_i] \gets\]
        %\[\frac{1}{N_{MC}} \sum_{j=1}^{N_{MC}} C_{(lb,ub)}%(D^{\psi}(G^{\theta}(y_i, z^{(j)}), y_i));
        %\]
        \State Compute generative loss:
        \[
        \begin{aligned}
        \mathcal{L}_G \gets &\frac{1}{B} \sum_{i=1}^{B} \|D^{\psi}(x_i, y_i) - \hat{\mathbb{E}}_{p_\theta}[D^{\psi}(X, y_i) \mid y_i]\|^2;
        \end{aligned}
        \]
        \State Update generator parameters:
        \[
        \theta \gets \text{Adam}(\mathcal{L}_G);
        \]
    \EndFor
\end{algorithmic}
\end{algorithm}

\newpage
\section{Properties of MCGAN}
In this section, we explore some favorable properties of MCGAN to further support why it can achieve such advantageous numerical results.
\begin{comment}
    \subsection{A non-optimal discrimnator with discriminability}
Here, we present an example of non-optimal discriminator with discriminability. In the case of BCE,
However, we can find a non-optimal discriminator with discriminability as follows:
\begin{equation}\label{eq:nonoptD}
\begin{aligned}
\hat{D}^{\phi^*_{\mu,\nu_{\theta}}}(x)=& D^{\phi^*_{\mu,\nu_{\theta}}}(x)\\&-\frac{1}{5}\left(D^{\phi^*_{\mu,\nu_{\theta}}}(x)-\frac{1}{2}\right)\mathbf{1}_{\{p_{\mu}(x)>p_{\nu_{\theta}}(x)\}}
\end{aligned}    
\end{equation}
It is easy to verify that $(D^{\hat{\phi}^*_{\mu,\nu_{\theta}}}(x)-1/2)(p_{\mu}(x)-p_{\nu_{\theta}}(x))>0,$
and ${D}^{\hat{\phi}^*_{\mu,\nu_{\theta}}}(x)\in (0,1)$ for all $x\in\mathcal{X}$ as permitted by the feasible set. Hence, it has discriminability, although it is not optimal.  We also present a $1$-D example in Figure \ref{optimalD}.
\begin{figure}[h]
\centering
\includegraphics[width=1\linewidth]{Figures/optimalD.pdf}
\caption{One illustration on the discriminability of the discriminator. In this example, $\mu\sim\mathcal{N}(-2,1),\nu_{\theta}\sim\mathcal{N}(
2,1)$.  The optimal discriminator is given by \eqref{optimalD} and the weaker version of discriminator $\hat D^{\phi^*}$ is defined in Eqn. \eqref{eq:true_D}.}
\label{optimalD}
\end{figure}
\end{comment}

\subsection{Relation to $f$-divergence}
In the case of vanilla GAN \cite{goodfellow2014generative},
 the optimal discriminator is given by
 \begin{equation}\label{eq:true_D}
D^{\hat{\phi}_{\mu,\nu_{\theta}}^*}(x)=\frac{p_{\mu}(x)}{p_{\mu}(x)+p_{\nu_{ \theta}}(x)}.
 \end{equation}
 Given \eqref{eq:true_D}, the vanilla GAN objective can be interpreted as minimizing \emph{Jensen-Shannon divergence} between $\mu$ and $\nu_{\theta}$, subtracting a constant term $\log(4)$. The generator is therefore 
trained to minimize the Jensen-Shannon divergence. Similarly, \cite{mao2017least} also showed that optimizing LSGANs yields minimizing the Pearson $\chi^2$ divergence between real and fake measures. 

Given \eqref{eq:true_D}, we are also able to explore the connection between MCGAN and $f$-divergence. As proven in Lemma \ref{theorem: f-div}, when considering the optimal discriminator \eqref{eq:true_D}, the difference between real and fake expected discriminator output in \eqref{eq:gradientMCGAN} can be interpreted as a $f$-divergence between $\mu$ and $\bar\nu$, where $\bar\nu:=\frac{(\mu+\nu)}{2}$ represents the averaged measure with density $p_{\bar\nu}(x):=\frac{p_{\mu}(x)+p_{\nu_{\theta}}(x)}{2}$. 

\begin{restatable}{lem}{fdiv}\label{theorem: f-div}
Given the optimal discriminator in \eqref{eq:true_D}, optimizing the MCGAN objective \eqref{eq:unobj} is equivalent to minimizing the square of $f$-divergence:
\begin{equation}
\nabla_{\theta}\mathcal{L}_G(\theta;\phi^*)=\nabla_{\theta}[\emph{Div}_{f}(\mu|\bar\nu)]^2,
\end{equation}
where $f(x)=x(x-1)$.
\end{restatable}

Lemma \ref{theorem: f-div} establishes a connection between MCGAN and $f$-divergence, illustrating the information-theoretic aspects of the MCGAN framework. Unlike the $\text{KL}(\nu|\mu)$ induced in non-saturating loss \cite{arjovsky2017towards}, this $\text{Div}_{f}(\mu|\bar\nu)$ can avoid mode dropping by assigning moderate cost on the occasions where $p_{\mu}(x)\gg p_{\nu_{\theta}}(x)$.  

\subsection{Improved stability}
%\noteNi{some of the following should be put into the intro}
%Incorporating MSE into the generative loss function has advantages more than just optimality. 
Lack of stability is a well-known issue in GAN training, and it arises due to several factors. \cite{arjovsky2017towards} provides insights into this instability issue of non-saturating loss. The instability is analyzed by modeling the inaccurate discriminator as an optimal discriminator perturbed by a centered Gaussian process. Given this noisy version of the optimal discriminator, it can be shown that the gradient of non-saturating loss follows a centered Cauchy distribution with infinite mean and variance, which leads to massive and unpredictable updates of the generator parameter. Hence it can be regarded as the source of training instability.

%By following the same idea,
In contrast, we prove that the proposed regression loss function in MCGAN can overcome the instability issue. Moreover, we relax the condition of the noise vector from the independent Gaussian distribution to a more general distribution.  %not restrict the random noise of the discriminator to be an independent Gaussian, which is a rather strong assumption. 

\begin{restatable}{thm}{stability}\label{theorem:stability}
Let $D^{\phi_{\epsilon}}(x)$ be a noisy version of optimal discriminator such that $D^{\phi_{\epsilon}}(x)=D^{\phi^*} (x)+\epsilon_1(x)$ and $\nabla_x D^{\phi_{\epsilon}}(x)=\nabla_x D^{\phi^*}(x) +\epsilon_2(x)$ for $ \forall x\in \mathcal{X},$ where $\epsilon_1(x)$ and $\epsilon_2(x)$ are two uncorrelated and centered random noises that are indexed by $x$ and have finite variance.\footnote{ This assumption of centered random noise is made due to the fact that as the approximation gets better, this
error looks more and more like centered random noise due to the finite precision \cite{arjovsky2017towards}.} Then for $\mathcal{L}_G(\theta;\phi)$ in \eqref{eq:unobj}, we have 
\begin{equation}
    \mathbb{E}[\nabla_{\theta}\mathcal{L}_G(\theta;\phi_\epsilon)]=\nabla_{\theta}\mathcal{L}_G(\theta;\phi^*),
\end{equation} and the variance of $\nabla_{\theta}\mathcal{L}_G(\theta;\phi_\epsilon)$ is finite and depends on the difference between $\mu$ and $\nu_{\theta}.$ Specifically, when $\nu_{\theta^*}=\mu$, we have $\nabla_{\theta}\mathcal{L}_G({\theta^*};\phi_\epsilon)=0$ almost surely.
\end{restatable}

Theorem \ref{theorem:stability} implies that given an inaccurate discriminator, the expected value of the gradient of $\mathcal{L}_G(\theta;\phi)$ in \eqref{eq:unobj} is the accurate gradient given by the optimal discriminator and its variance is finite. More importantly, its variance is determined by the discrepancy between $\mu$ and $\nu_{\theta}$, specifically when the fake measure produces the real measure, the gradient is zero almost surely, indicating improved training stability.

\subsection{Relation to feature matching}
In order to enhance the generative performance, a feature matching approach is proposed in \cite{salimans2016improved} which adds to the generative loss function an additional cost that matches the statistic of the real and generated samples given by the activation on an intermediate layer of the discriminator. The generator hence is trained to generate fake samples that reflect the statistics (features maps) of real data rather than just maximizing its discriminator outputs. 

To be specific, suppose we have a feature map $\psi$ that maps each $x\in\mathcal{X}$ to  a feature vector $\psi(x)=(\psi_1(x),\psi_2(x),\ldots,\psi_n(x))\in\mathbb{R}^n$ where each $\psi_i\in C_{b}(\mathcal{X})$, then the feature matching approach adds to the generative loss function an additional cost defined as:
\begin{equation}\label{eq:fm}
    R_{\text{fm}}(\theta;{\psi})=\|\mathbb{E}_{\mu}[\psi(x)]-\mathbb{E}_{\nu_{\theta}}[\psi(x)]\|^2_{2}.
\end{equation}

Although empirical results indicate that feature matching is effective, it lacks a theoretical guarantee that minimizing the difference of features can help us reach the Nash equilibrium or optimality $\nu_{\theta}=\mu$. Hence \eqref{eq:fm} is commonly used as a regularization term rather than an individual loss function like the one proposed in our MCGAN.

In the case of MCGAN, if we construct the discriminator as a linear transformation of the feature map, i.e. $D^{\phi}=\phi^T(\psi(x),\mathbf{1})$ where $\phi\in\mathbb{R}^{n+1}$ is a linear functional. Given the novel generative loss function in \eqref{eq:unobj}, we have
\[\nabla_{\theta} \mathcal{L}_{G}(\theta;\phi)=\nabla_{\theta} 
|\mathbb{E}_{\mu}[\phi^T(\psi(x),\mathbf{1})]-\mathbb{E}_{\nu_{\theta}}[\phi^T(\psi(x),\mathbf{1})]|^2.
\]
Here, the generator is also trained to match the feature maps of real samples and fake samples, but in a weighted average way.  By using the linear transformation on the feature map, the discriminator is trained to focus on the most relevant features and assign them larger weights while assigning relatively smaller weights to less important features. As a result, the generator is trained to match the feature maps more efficiently.

\newpage

\section{Proofs}
\subsection{Proof of optimality in Theorem \ref{theorem:g}}
%\theoremg*
\begin{proof}
For every $\phi\in\Phi$, the derivative of $L_{G}(\theta;\phi,\mu)$ in \eqref{eq:unobj} w.r.t $\theta$ can be derived as 
$$
\begin{aligned}
\nabla_{\theta}L_{G}(\theta;\phi,\mu)&=\mathbb{E}_{\mu}[2(D^{\phi} (x)-\mathbb{E}_{\nu_\theta}[D^{\phi} (x)])H(\theta,\phi)],\\
&=2(\mathbb{E}_{ \nu_\theta }[D^{\phi} (x)]-\mathbb{E}_{\mu }[D^{\phi} (x)])H(\theta,\phi),\\
\end{aligned}
$$
where $$
H(\theta,\phi)={\mathbb{E}}_{z\sim p(z)}[(\nabla_{\theta}G^{\theta}(z))^T\cdot \nabla_x D^\phi(G^{\theta}(z))]
.$$

 If $\theta^*$ is a local minimizer of $L_{G}(\theta;\phi,\mu)$, then by first-order condition, it satisfies that
\begin{equation}\label{eq:firstorder}
\mathbb{E}_{ \nu_{\theta^*} }[D^{\phi} (x)]-\mathbb{E}_{\mu }[D^{\phi} (x)] =0,
\end{equation}
or
\begin{equation}\label{eq: Hthetaphi}
H(\theta^*,\phi)=\vec{0}.
\end{equation}
By Assumption \ref{assumption:H}, equation \eqref{eq: Hthetaphi} holds only if  $\nu_{\theta^*} =\mu$. Here we focus on the other case \eqref{eq:firstorder}.

Given a parameterization map $\phi'_{\cdot,\cdot}:\mathcal{P(X)}\times\mathcal{P(X)}\rightarrow\Phi$ and $D^{\phi'_{\cdot,\cdot}}\in \mathcal{D}_{\text{Dis}}$, if $\theta^*$ is a local minimizer of $L_{G}(\theta;\phi'_{\mu,\nu_{\theta}},\mu)$ , we must have 

\begin{equation}\label{eq:theta*}
\mathbb{E}_{ \nu_{\theta^*} }[D^{\phi'_{\mu,\nu_{\theta^*}}} (x)]-\mathbb{E}_{\mu }[D^{\phi'_{\mu,\nu_{\theta^*}}} (x)] =0.
\end{equation}
Since $D^{\phi'_{\cdot,\cdot}}\in \mathcal{D}_{\text{Dis}},$ without generality, we set $a=1$ and $c=0$ and have
\begin{equation}
\begin{aligned}
&\mathbb{E}_{ \nu_{\theta} }[D^{\phi'_{\mu,\nu_{\theta}}} (x)]-\mathbb{E}_{\mu }[D^{\phi'_{\mu,\nu_{\theta^*}}} (x)] \\
=& \int_{\mathcal{A}^{\mu,\nu_{\theta}}} D^{\phi'_{\mu,\nu_{\theta}}} (x) (p_{\mu}(x)-p_{\nu_{\theta}}(x))dx\\
>&0,\\ 
\end{aligned}
\end{equation}
for every different $\mu$ and $\nu_\theta$. Hence equality \eqref{eq:theta*} holds if and only if $\nu_{\theta^*}=\mu,$ which completes the proof.
\end{proof}

\subsection{Proof of $f$-divergence in Lemma \ref{theorem: f-div}} 

\begin{proof}
Given the optimal discriminator in \eqref{eq:true_D}, we have 
$$
\begin{aligned}
&\mathbb{E}_{\mu}[D^{\phi^*}(x)]-\mathbb{E}_{\nu_\theta}[D^{\phi^*}(x)]\\
&= \int\limits_{\supp \mu \cup \supp\nu_{\theta}} \frac{p_{\mu}(x)}{p_{\mu}(x)+p_{\nu_{\theta}}(x)}(p_{\mu}(x)-p_{\nu_{\theta}}(x))dx.\\
\end{aligned}
$$
Let $\bar\nu:=\frac{\mu+\nu_{\theta}}{2}$ be the averaged measure defined on ${\supp \mu \cup \supp\nu_{\theta}}$, then we have 
$$
\begin{aligned}
&\mathbb{E}_{\mu}[D^{\phi^*}(x)]-\mathbb{E}_{\nu_\theta}[D^{\phi^*}(x)]\\
&= \int\limits_{\supp \mu \cup \supp\nu_{\theta}} \frac{p_{\mu}(x)}{2p_{\bar\nu}(x)}\frac{(p_{\mu}(x)-p_{\nu_{\theta}}(x))}{p_{\bar\nu}(x)}p_{\bar\nu}(x)dx\\
&= \int\limits_{\supp \mu \cup \supp\nu_{\theta}} \frac{p_{\mu}(x)}{p_{\bar\nu}(x)}
\left(\frac{p_{\mu}(x)}{p_{\bar\nu}(x)}-1\right)p_{\bar\nu}(x)dx\\
&= \int\limits_{\supp \mu \cup \supp\nu_{\theta}} f\left(\frac{p_{\mu}(x)}{p_{\bar\nu}(x)}\right)p_{\bar\nu}(x)dx\\
&= \text{Div}_{f}(\mu\|\bar 
\nu),\\
\end{aligned}
$$
where $f(x):=x(x-1)$ is a convex function and $f(1)=0$. Therefore, $\text{Div}_{f}(\mu\|\bar 
\nu)$ is well-defined $f$-divergence. Furthermore, we can observe that the gradient of the generator objective function $L_G(\theta;\phi^*)$ can be written as the gradient of the squared $f$-divergence:
$$
\begin{aligned}
\nabla_{\theta}L_G(\theta;\phi^*)&=\nabla_{\theta}\mathbb{E}_{\mu}\left[D^{\phi^*}(x)-\mathbb{E}_{\nu_\theta}[D^{\phi^*}(x)]\right]^2\\
&=\nabla_{\theta}\left[\mathbb{E}_{\mu}[D^{\phi^*}(x)]-\mathbb{E}_{\nu_\theta}[D^{\phi^*}(x)]\right]^2\\
&=\nabla_{\theta}[\text{Div}_{f}(\mu|\bar\nu)]^2,
\end{aligned}
$$
which completes the proof.
\end{proof}
\subsection{Proof of improved stability in Theorem \ref{theorem:stability}}

\begin{proof}
Since $D^{\phi_{\epsilon}}(x)=D^{\phi^*} (x)+\epsilon_1(x)$ and $\nabla_x D^{\phi_{\epsilon}}(x) = \nabla_x D^{\phi^*}(x) +\epsilon_2(x)$, we have 
$$
\begin{aligned}
&\nabla_{\theta}L_G(\theta;\phi_{\epsilon})=\left(\mathbb{E}_{\nu_\theta}[D^{\phi_{\epsilon}}(x)]-\mathbb{E}_{\mu}[D^{\phi_{\epsilon}}(x)]\right)H(\theta,\phi_{\epsilon})\\
=& \left(\Delta(\theta,\phi^*)+\bar\epsilon_1(\theta)\right)H(\theta,\phi_{\epsilon}),\\
\end{aligned}
$$
where
$$
\Delta(\theta,\phi)=\mathbb{E}_{\nu_\theta}[D^{\phi}(x)]-\mathbb{E}_{\mu}[D^{\phi}(x)]
$$
and
$$\bar\epsilon_1(\theta)=\mathbb{E}_{\nu_\theta}[\epsilon_1(x)]-\mathbb{E}_{\mu}[\epsilon_1(x)].$$
Because 
$$
\begin{aligned}
H(\theta,\phi_{\epsilon})=&\mathbb{E}_{z\sim \mu_z}[(\nabla_{\theta}G^{\theta}(z))^T\cdot \nabla_x D^{\phi_{\epsilon}}(G^{\theta}(z))]\\
=&H(\theta,\phi^*)+\mathbb{E}_{z\sim \mu_z}[(\nabla_{\theta}G^{\theta}(z))^T\cdot \epsilon_2(x)]\\
=&H(\theta,\phi^*)+\bar{\epsilon}_2 (\theta),\\
\end{aligned}
$$
we have  
$$
\begin{aligned}
\nabla_{\theta}L_G(\theta;\phi_{\epsilon})=&\nabla_{\theta}L_G(\theta;\phi^*)+\bar\epsilon_1(\theta)H(\theta,\phi^*)\\&+\Delta(\theta,\phi^*)\bar\epsilon_2(\theta)+\bar\epsilon_1(\theta)\bar\epsilon_2(\theta).
\end{aligned}
$$
Since both $\bar\epsilon_1(\theta)$ and $\bar\epsilon_2(\theta)$ are weighted averages or linear combinations of centered random noises, they are both centered noises as well. Moreover, the expectation of $\bar\epsilon_1(\theta)\bar\epsilon_2(\theta)$ is also zero since $\epsilon_1(x)$ and $\epsilon_2(x)$ are uncorrelated. Hence the mean of $\nabla_{\theta}L_G(\theta;\phi_{\epsilon})$ equals to $\nabla_{\theta}L_G(\theta;\phi^*).$ By the definition of $\Delta(\theta,\phi)$ and $\bar\epsilon_1(\theta)$, its variance also depends on the difference between $\mu$ and $\nu_{\theta}$, which completes the proof.
\end{proof}

\subsection{Discriminability of different GAN variants}\label{appendix:discriminability}
Here we provide the  of the discriminability of optimal discriminators in these GAN variants described in Table \ref{tab:discriminability}. 

\begin{itemize}
    \item \textbf{Vanilla GAN \cite{goodfellow2014generative}}: GAN employs BCE as the discriminative loss function defined as
\begin{equation}\label{eq:bce}
\begin{aligned}
L_{D}(\phi;\mu,\nu_{\theta})=&\mathbb{E}_{\mu}\left[\log( D^\phi(X))\right]\\&+\mathbb{E}_{\nu_\theta}\left[\log(1-D^\phi(X))\right].    
\end{aligned}
\end{equation}

As proven in \cite{goodfellow2014generative}, the optimal discriminator given binary cross-entropy loss can be derived as: 
 \begin{equation}\label{optD_BCE}
D^{\phi^*_{\mu,\nu_{\theta}}}(x)=\frac{p_{\mu}(x)}{p_{\mu}(x)+p_{\nu_{ \theta}}(x)}.
 \end{equation}
 Let us consider function $f(l)=\frac{1}{1+l}$ for $l>0$. Notice that $f(l)>1/2$ when $l<1$, and $f(l)<1/2$ when $l>1$. Also notice that $D^{\phi^*_{\mu,\nu_{\theta}}}(x)=f(\frac{p_{\nu_{\theta}}(x)}{p_{\mu}(x)})$, it is easy to verify that $(D^{\phi^*_{\mu,\nu_{\theta}}}(x)-1/2)(p_{\mu}(x)-p_{\nu_{ \theta}}(x))>0$ when $p_{\mu}(x)\neq p_{\nu_{\theta}}(x).$

\item \textbf{Least Square GAN \cite{mao2017least}}: LSGAN employs least square loss function defined  as follows:
$$
\begin{aligned}
L_{D}(\phi;\mu,\nu_{\theta})=&-\mathbb{E}_{\mu}\left[(D^\phi(X)-\alpha)^2\right]\\&-\mathbb{E}_{\nu_\theta}\left[( D^\phi(X)-\beta)^2\right],
\end{aligned}
$$
where $\alpha,\beta\in\mathbb{R}$, and $\alpha\neq\beta$. The optimal discriminator is given as 
 \begin{equation}\label{optD_LS}
D^{\phi^*_{\mu,\nu_{\theta}}}(x)=\frac{\alpha p_{\mu}(x)+\beta p_{\nu_{ \theta}}(x)}{p_{\mu}(x)+p_{\nu_{ \theta}}(x)},
 \end{equation}
 Similarly, by the fact that $D^{\phi^*}(x)=f(\frac{p_{\nu_{ \theta}}(x)}{p_{\mu}(x)})$, where $f(l)=\frac{\alpha +\beta l}{1+l}$, we can verify that this discriminator also has (strict) discriminability in the sense that $\sign(\alpha-\beta)(D^{\phi^*_{\mu,\nu_{\theta}}}(x)-\frac{\alpha+\beta}{2})(p_{\mu}(x)-p_{\nu_{\theta}}(x))>0$ when $p_{\mu}(x)\neq p_{\nu_{\theta}}(x).$

\item \textbf{Geometric GAN \cite{lim2017geometric}}:  Hinge loss function is defined as
$$
\begin{aligned}
L_{D}(\phi;\mu,\nu_{\theta})=&-\mathbb{E}_{\mu}\left[\max(0,1-D^\phi(X))\right]\\&-\mathbb{E}_{\nu_\theta}\left[\max(0,1+D^\phi(X))\right].
\end{aligned}
$$
By Lemma B.1 in \cite{lim2017geometric}, it is straightforward to show that the optimal discriminator can be derived as:
 \begin{equation}\label{eq:hinge}
D^{\phi^*_{\mu,\nu_{\theta}}}(x)= 2\mathds{1}_{\{p_{\mu}(x)\geq p_{\nu_{\theta}}(x)\}}-1.
 \end{equation}
It is clear that $D^{\phi^*}(x)=f(\frac{p_{\mu}(x)}{p_{\nu_{ \theta}}(x)})$, where $f(l)=2\mathds{1}_{\{l\geq 1\}}-1$, and  $D^{\phi^*}(x)(p_{\mu}(x)-p_{\nu_{ \theta}}(x))>0$ when $p_{\mu}(x)\neq p_{\nu_{\theta}}(x).$

\item \textbf{Energy-based GAN \cite{zhao2016energy}}:  Energy-based loss function is defined as
$$
\begin{aligned}
L_{D}(\phi;\mu,\nu_{\theta})=&- \mathbb{E}_{\mu}\left[D^\phi(X)\right]\\&-\mathbb{E}_{\nu_\theta}\left[\max(0,m-D^\phi(X))\right].
\end{aligned}
$$
where $m>0$. By Lemma 1 in \cite{zhao2016energy}, the optimal discriminator given energy-based loss function 
can be derived as:
 \begin{equation}\label{eq:energe}
D^{\phi^*_{\mu,\nu_{\theta}}}(x)= m\mathds{1}_{\{p_{\mu}(x)<p_{\nu_{\theta}}(x)\}}.
 \end{equation}
It is straight forward to verify that  $-m(D^{\phi^*_{\mu,\nu_{\theta}}}(x)-m/2)(p_{\mu}(x)-p_{\nu_{\theta}}(x))>0$ when $p_{\mu}(x)\neq p_{\nu_{\theta}}(x).$ 

\item \textbf{$f$-GAN \cite{nowozin2016f}}:
In $f$-GAN, variational lower bound (VLB) on the $f$-divergence $\text{Div}_{f}(\mu||\nu_{\theta})$ is used in the generative-adversarial approach to mimic the target distribution $\nu_{\theta}$. Let $f:\mathbb{R}_{+}\to\mathbb{R}$ be a convex, lower-semicontinuous function. In $f$-GAN, the discriminative loss is defined as the variational lower bound  on certain $f$-divergence:
 \begin{equation}\label{eq:fdiv}
\begin{aligned}
    L_{D}(\phi;\mu,\nu_{\theta})=&\mathbb{E}_{\mu}\left[D^\phi(X)\right]\\&-\mathbb{E}_{\nu_\theta}\left[f^*(D^\phi(X))\right],
\end{aligned}
 \end{equation}
where $f^*$ is the convex conjugate function of $f$. Under mild conditions on function $f$ \cite{nguyen2010estimating}, the maximum of \eqref{eq:fdiv} is achieved when 
 \begin{equation}\label{eq:fgan}
D^{\phi^*_{\mu,\nu_{\theta}}}(x)= f'\left(\frac{p_{\mu}(x)}{p_{\nu_{\theta}}(x)}\right),
 \end{equation}
where $f'$ is the first order derivative of $f$ and increasing due to the convexity of $f$. Consequently, we can choose $c=f'(1)$ such that  $(D^{\phi^*_{\mu,\nu_{\theta}}}(x)-c)(p_{\mu}(x)-p_{\nu_{\theta}}(x))>0$ when $p_{\mu}(x)\neq p_{\nu_{\theta}}(x).$ A more detailed list of $f$-divergence can be found in \cite{nowozin2016f}.
\end{itemize}

\section{Evaluation metrics}\label{sec:con_numeric}
In this section, we provide detailed introduction to those test metrics used in our numerical experiment.

\subsection{Evaluation metrics of image generation}

To assess the quality of images generated, we employ three quality metrics
\emph{Inception Score} (IS), \emph{Fr\'echet Inception Distance} (FID), and
\emph{Intra Fr\'echet Inception Distance} (IFID) together with two recognizability
metrics \emph{Weak Accuracy} (WA) and \emph{Strong Accuracy} (SA).

Inception Score \cite{salimans2016improved} (IS) is a popular metric
to evaluate the variety and distinctness of the generated images. It is given by
\begin{equation}
     \text{IS}=\exp\{\mathbb{E}_{X\sim
     \nu_\theta}\left[D_{\text{KL}}(P(Y|X)||P(Y))\right]\},
\end{equation} where $D_{\text{KL}}$ is the KL-divergence between the
conditional class distribution $P(Y|X)$ and marginal class distribution of the
$P(Y)=\mathbb{E}[P(Y|\mathbf{X})]$. The conditional class distribution $P(Y|X)$
is computed by \emph{InceptionV3} network pre-trained on ImageNet. The higher
IS, the better the quality. By the definition, the IS does not consider real
images, so cannot measure how well the fake measure induced by the generator is
close to the real distribution. Other limitations, as noted in
\cite{barratt2018note}, are: high sensitivity to small changes in weights of the
Inception network, and large variance of scores. To consider both diversity and
realism, the following FID and IFID are employed as well.

Fr\'echet Inception Distance \cite{heusel2017gans} (FID) compares the
distributions of Inception embeddings of real and generated images, denoted by
$p_{d}$ and $p_{\theta}$ respectively. Under the assumption that the features of
images extracted by the function $f$ are of multivariate normal distribution.
The FID score of $p_{\theta}$ w.r.t $p_{d}$  is defined as
\begin{align}\label{eq:fid}
\text{FID}(p_{d},p_{\theta})=&\|\mathbf{\mu}_r-\mathbf{\mu}_g\|^2\\&+\text{Tr}(\mathbf{\Sigma}_r+\mathbf{\Sigma}_g-2(\mathbf{\Sigma}_r\mathbf{\Sigma}_g)^{\frac{1}{2}}), \notag
\end{align}

where $(\mathbf{\mu}_r,\mathbf{\Sigma}_r)$ and
$(\mathbf{\mu}_g,\mathbf{\Sigma}_g)$ denote the mean and covariance matrix of
the feature of real and generated image distribution respectively. Given a
data-set of images $\{x_i\}^N_{i}$ and the Inception embedding function $f$, the
Gaussian parameters $(\mathbf{\mu}_r,\mathbf{\Sigma}_r)$ are  then approximated
as 
\begin{align*}
\mathbf{\mu}&=\frac{1}{N}\sum^N_{i=0}f(x_i),\\
\Sigma&=\frac{1}{N-1}\sum_{i=0}^N(f(x^{(i)})-\mathbb{\mu})(f(x^{(i)})-\mathbb{\mu})^T.
\end{align*}

We can see from \eqref{eq:fid} that FID directly compares the distribution of
features of real and fake images. However, the Gaussian assumption made in FID
computation might not be met in practice. Also, FID has high sensitivity to the
sample size — a small size might cause over-estimation of the real FID.

Intra Fr\'echet Inception Distance \cite{devries2019evaluation} (IFID) is used to quantify intra-class
diversity. It is defined as the average of conditional FID given every class
$y\in\mathcal{Y}$, i.e., \[
\text{FID}(p_{d},p_{\theta})=\frac{1}{|\mathcal{Y}|}\sum_{y\in\mathcal{Y}}\text{FID}(p_{d}(y),p_{\theta}(y)),
\] where 
\begin{eqnarray*}&&\text{FID}(p_{d}(y),p_{\theta}(y))= 
\|\mathbf{\mu}_r(y)-\mathbf{\mu}_g(y)\|^2\\&&+\text{Tr}(\mathbf{\Sigma}_r(y)+\mathbf{\Sigma}_g(y)-2(\mathbf{\Sigma}_r(y)\mathbf{\Sigma}_g(y))^{\frac{1}{2}}).
\end{eqnarray*}

The combination of IS, FID, and IFID provides a comprehensive evaluation for generated image quality assessment. IS and FID are measured between 50K generated images given 10 different random seeds in this paper, and IFID is the average intra-class results of FID.

Recognizability is as crucial as realism and diversity in a good image
generative model, therefore two classification accuracy are adopted—a weak accuracy (WA)
measured by a two-layer convolutional neural network
\footnote{https://pytorch.org/tutorials/beginner/blitz/cifar10\_tutorial.html}
and a strong accuracy (SA) by the ResNet-50 \cite{he2016deep}. Both
classifiers are pre-trained on the same training set as for the generative
model. The WA discerns subtle differences for better inter-model comparison,
while the SA is more accurate for intra-model latent space analysis.

\subsection{Evaluation metrics of timeseries generation}
In the following context, we describe the definition of the test metrics precisely, more detailed discussions can be found in \cite{liao2024sig,xiao2023signature}. %Let $(X_{t})_{t=1}^{T}$ denote a $d$-dimensional time series sampled from the real target distribution. We first extract the past/future pairs $(X_{t-p+1:t}, X_{t+1:t+q})_{t \in \mathcal{T}}$, where $\mathcal{T}$ is the set of time indexes. Given the generator $G^{\theta}$, for each input sample $(X_{t-p+1: t})$, we generate one sample of the $q$-step forecast $\hat{X}^{(t)}_{t+1, t+q}$.  The synthetic data generated by $G^{\theta}$ is given by $\{\hat{X}^{(t)}_{t+1, t+q}\}_t$, which we use to compute the test metrics. 

\begin{itemize}
\item \textbf{ABS metric on marginal distribution}: 
The \textbf{ABS metric} is a histogram-based distributional metrics where we compare the empirical density function (epdf) of real data and synthetic data. When talking about epdf, we mean each bin's raw count divided by the total number of counts and the bin width. For each feature dimension $i \in \{1, \cdots, d\}$, we denote the epdfs of real data and synthetic data as $\hat{df}_{r}^{i}$ and $\hat{df}_{G}^{i}$ respectively. Here the epdfs of synthetic date $\hat{df}_{G}^{i}$ is computed on the bins derived from the histogram of real data. The ABS metric is defined as the absolute difference of those two epdfs averaged over feature dimension, i.e.
\begin{align*}
\frac{1}{d}\sum_{i = 1}^{d} \vert \hat{df}_{r}^{i} - \hat{df}_{G}^{i} \vert_{1},
\end{align*} 
where $\vert \hat{df}_{r}^{i} - \hat{df}_{G}^{i} \vert_{1}$ is computed as the $l_1$ distance between the epdfs of real and synthetic data on each bin. Notice that although the ABS metric cannot give a fully point-separating metric on the space of measure, it can still provide a general description of the similarity between two set of data given a reasonable number of bins. Considering the computational cost, we set number of bins to 50 in our implementation.

\item \textbf{ACF metric on temporal dependency}: We use the absolute error of the auto-correlation estimator by real data and synthetic data as the metric to assess the temporal dependency and name it as \textbf{ACF metric}. For each feature dimension $i \in \{1, \ldots, d\}$, we compute the auto-covariance of the $i^{th}$ coordinate of time series data $X$ at lag $\tau$ under real measure and synthetic measure resp, denoted by $\rho_r^{i}(\tau)$ and $\rho_{G}^{i}(\tau)$. Then the estimator of the lag-$\tau$ auto-correlation of the real/synthetic data is given by $\frac{\rho_{r}^{i}(\tau)}{\rho_{r}^{i}(0)}$/ $\frac{\rho_{G}^{i}(\tau)}{\rho_{G}^{i}(0)}$. 
The ACF metric is defined to be the absolute difference of  auto-correlation up to lag $\tau$ given as follows:
\begin{align*}
\frac{1}{d\tau}\sum_{k = 1}^{\tau}\sum_{i = 1}^{d}\left \vert \frac{\rho_{r}^{i}(k)}{\rho_{r}^{i}(0)} -  \frac{\rho_{G}^{i}(k)}{\rho_{G}^{i}(0)}\right \vert.
\end{align*}
%In addition, we present the ACF plot, which illustrates the autocorrelation of each
%coordinate of the time series with different lag values. The synthetic data’s quality is
%evaluated by how closely its ACF plot resembles that of the real data, as it indicates
%the synthetic data’s ability to capture long-term temporal dependencies.

\item \textbf{Corr metric on feature dependency}: For $d>1$, we assess the feature dependency by using the $l_1$ norm of the difference between cross-correlation matrices and name it as \textbf{Corr metric}. To be specific, let $\tau^{i, j}_{r}$ and $\tau^{i, j}_{G}$ denote the correlation of the $i^{th}$ and $j^{th}$ feature of time series under real measure and synthetic measure resp.
The correlation metric between the real data and synthetic data is given by $l_{1}$ norm of the difference between two correlation matrices, i.e.
\begin{align*}
\frac{1}{d^2}\sum_{i = 1}^{d}\sum_{j = 1}^{d}|\tau^{i, j}_{r} - \tau^{i, j}_{G}|.
\end{align*}

\item \textbf{$R^2$ error for usefulnese}: In order to be useful, the synthetic data should inherit the predictive characteristics of the original, meaning that the synthetic data should be just as useful as the real data when used for the same predictive purpose (i.e. train-on-synthetic, test-on-real). To measure the usefulness of the synthetic data, we follow \cite{esteban2017real, yoon2019time} and consider the problem of predicting next-step temporal vectors using the lagged values of time series using the real data and synthetic data. First, we train a supervised learning model on real data to predict next-step values and evaluate it in terms of $R^{2}$ (TRTR). Then we train the same supervised learning model on synthetic data and evaluate it on the real data in terms of $R^{2}$ (TSTR). The closer two $R^{2}$ are, the better the generative model is. The predictive score is then defined as the \textbf{$R^{2}$ relative error}. This test metric is reasonable because it demonstrates the ability of the synthetic data to be used
for real applications.

\end{itemize}

\section{Architectures, hyperparameters, and training techniques}\label{appendix:numdetail}

\subsection{Image generation}
\subsubsection{Backbone architectures}
We employ BigGAN
\cite{brock2018large} and cStyleGAN2 \cite{karras2020analyzing} architectures as
the backbones for image generation experiments. BigGAN \cite{brock2018large}, as a member of projection-based cGAN,
is a collection of recent best practices in conditional image generation, and it
is widely used due to its satisfactory generation performance on high-fidelity
image synthesis. We use the BigGAN architecture with the same regularization
methods like \emph{Exponential Moving Averages} (EMA)
\cite{karras2017progressive} and \emph{Spectral Normalization}
\cite{miyato2018spectral} have already been adopted. We adopt BigGAN's PyTorch
implementation \footnote{\url{https://github.com/PeterouZh/Omni-GAN-PyTorch}}
and shows the architectural details in Table \ref{table_CIFAR10_network} for
completeness.   

\begin{table}[htbp]
\centering
\resizebox{\columnwidth}{!}{%
\begin{subtable}[htbp]{0.7\columnwidth}
\centering
\begin{tabular}{c}
\toprule
$z\in \mathbb{R}^{128}\sim \mathcal{N}(0,I)$ 
\\\midrule
SNLinear $ 128 \to 4\times4\times4ch$
\\\midrule
GResBlock up $4ch\to 4ch$
\\\midrule
GResBlock up $4ch\to 4ch$
\\\midrule
GResBlock up $4ch\to 4ch$
\\\midrule
BN, ReLU, $3\times 3$ SNConv $4ch\to3$
\\\midrule
Activation: Tanh
 \\\bottomrule
\end{tabular}
\caption{Generator}
\end{subtable}
\hfill
\begin{subtable}[htbp]{0.7\columnwidth}
\centering
\begin{tabular}{c}
\toprule
RGB image $x\in  \mathbb{R}^{32\times 32\times 3}$ 
\\\midrule
DResBlock down $3 \to 4ch$
\\\midrule
DResBlock down $4ch\to 4ch$
\\\midrule
DResBlock down $4ch\to 4ch$
\\\midrule
DResBlock down $4ch\to 4ch$
\\\midrule
SumPooling
\\\midrule
SNLinear $4ch\to 1$, $\text{embed}(y)\in \mathbb{R}^{256}$
 \\\bottomrule
\end{tabular}
\caption{Discriminator}
\end{subtable}
}
\caption{BigGAN architecture used in CIFAR-10 and CIFAR-100experiment, where $ch$ is set as $64$.}\label{table_CIFAR10_network}
\end{table}

cStyleGAN2 \cite{zhao2020differentiable}, is an improved and conditional version of
the original StyleGAN, is a generative adversarial network (GAN) architecture
designed for creating high-quality, diverse images. It addresses artifacts,
enhances image quality, and has been widely used for generating realistic
portraits and artwork. We adopt the code from the Github repository in
\cite{zhao2020differentiable}
\footnote{\url{https://github.com/mit-han-lab/data-efficient-gans/tree/master}}.
and use the default hyperparameter setting. The only
difference is the minor modification we made to incorporate our MC method.

\subsubsection{Loss functions} Since our MCGAN replaces the original generative
loss in Eqn. \eqref{orign_gloss} with the regression loss in Eqn.
\eqref{eq:unobj} and keeps the discriminator loss, we use two popular
discriminator's loss functions as baselines: the Hinge loss baseline and the BCE loss baseline.

\subsubsection{Hyperparameters} For the CIFAR-10 experiment, the batch size is set
to 32. We adopt the Adam optimizer in all experiments, with betas being 0.0 and
0.999. For both of the generator and discriminator, the learning rates are set to
0.0002 and the weight decay is 0.0001. The model is updated by using \emph{Exponential Moving Average} starting
after the first 5000 iterations. The generator is updated once every 3 times the
discriminator is updated. For the CIFAR-100 experiment, we have fewer training
samples for each class, so we update the discriminator 4 times per generator
training as a more accurate discriminator is needed. Each experiment is conducted on one Quadro RTX 8000 GPU. All experiments are conducted using fixed and default seed settings in the code base. 

For the large-scale and high-resolution experiments,  our MCGAN uses the Monte Carlo sample size $M=10$ for the LSUN bedroom dataset, and $M=4$ for FFHQ256 and ImageNet64 datasets. Their detailed settings are provided in the code released.

\subsubsection{Leaky Clamp } To stabilize the training of our regression loss, we employ the \emph{Leaky Clamp}
function to limit the discriminator output in a reasonable range so that the distance between fake and real discriminator output
will not exceed a predetermined range. The leaky clamp function is defined as \[
C_{(lb,ub)}(x)=\begin{cases} lb+\alpha (x-lb)\quad&\text{if $x\leq lb$}\\
x\quad&\text{if $lb\leq x\leq ub$}\\ ub+\alpha (x-ub)\quad&\text{if $ub<x$}
\end{cases} \] where $\alpha\in (0,1)$ is a small slope for values outside the
range $[lb,ub]$. Just similar to the negative slope in Leaky ReLU
\cite{maas2013rectifier}, the $\alpha$ in the Leaky Clamp function is used to
prevent from vanishing gradient problem. And by applying this Leaky Clamp on the
discriminator output when computing the regression loss, we are able to mitigate the early
collapse problem.

\paragraph{Data augmentation} To alleviate the overfitting and improve the
generalization on the small training set, especially for CIFAR-100 where each
class has scarce samples, we increase data efficiency by using the
\emph{Differentiable Augmentation} (DiffAug) \cite{zhao2020differentiable} which
imposes various types of augmentations on real and fake samples
\cite{zhao2020differentiable,karras2020training,zhang2019consistency}. We adopt \emph{Translation + Cutout} policy as suggested in \cite{zhao2020differentiable}. Besides, we also apply horizontal flips when loading the training dataset as in \cite{kang2021rebooting}.

\subsection{Video generation}
\subsubsection{Backbone architecture} For  both the generator and discriminator, the backbone we used in conditional video generation task  is called convolutional LSTM (ConvLSTM) unit proposed by \cite{shi2015convolutional} due to its effectiveness in video prediction tasks. For both generator and discriminator, the number of layers is 2 and the hidden dimension is specified as 64. The convolutional kernel size is set as (3,3) with padding (1,1). The activation function used is ReLU. In the model training, the generator takes in 5 past frames as the input and generates the corresponding 1-step future frame, then the real past frames and the generated future frames are concatenated along time dimension and put into the discriminator.
 
\subsubsection{Hyperparameter}
For moving MNIST dataset, the batch size is set to 16. And the frame size is downsampled from 64 to 32 and MC size is specified as 4 to reduce GPU memory consumption. The generator is updated every time the discriminator is updated.  The Adam optimizer is adopted, with betas being 0.5 and
0.999 and learning rate being 2e-4. 

\subsection{Timeseries generation}
\subsubsection{Backbone architecture}  The RNN model we employed in conditional timeseries generation task is built up using the AR-FNN architecture introduced in \cite{liao2024sig}. The AR-FNN is defined as a composition of PReLUs, residual layers and affine transformations. Its inputs are the past $p$-lags of the $d$-dimensional process we want to generate as well as the $d$-dimensional noise vector. A formal definition can be found in \cite{liao2024sig}. 

\subsubsection{Hyperparameter} For both low-dim VAR and stock datasets, the hidden dimension of AR-FNN is set as 50 with 3 number of layers. To increase model's capacity in generating high-dim time-series, we increased the hidden dimension to 128 for $d=10,50,100$ in VAR experiment. Similar to image generation experiment, the Adam optimizer is used with betas being (0, 0.9) and learning rate being 2e-4 for both discriminator and generator. The number of total training epochs is 1000 with batch size specified as 100.  The generator is updated every 4 times the discriminator is updated.  For all time-series experiment, we set MC size as 1000 due to less GPU memory consumed for time-series data.

\section{Supplementary numerical results}
In this section, we present the supplementary numerical results on both image generation and  timeseries generation.
\subsection{Image generation}
\subsubsection{Sensitivity analysis of MC sample size}
In our experiments, we use a Monte Carlo sample size ($M$) of 10 for BigGAN and 4 for StyleGAN2 as
outlined. We've conducted a sensitivity analysis using BigGAN on CIFAR-10 datasets. The results can be found in Table \ref{tab:mc_sensitivity}.

%\newcolumntype{Y}{>{\centering\arraybackslash}X}
\begin{table}[!ht]
\centering
\resizebox{\columnwidth}{!}{%
\begin{tabularx}{0.7\textwidth}{c| Y|Y| Y|Y|Y|Y}
\toprule
\multicolumn{1}{l|}{} & \multicolumn{1}{c|}{w/o MC} &   \multicolumn{1}{c|}{{$M$}=5} &
\multicolumn{1}{c|}{{$M$}=10}&
\multicolumn{1}{c|}{{$M$}=15} &   \multicolumn{1}{c|}{{$M$}=20} & \multicolumn{1}{c}{{$M$}=25}
\\\midrule
FID $\downarrow$   & 4.34 &  3.73& 3.55& 3.61&  4.99& 4.50
\\\midrule
IS $\uparrow$   & 9.41 & 9.69 & 9.96 & 10.07 & 10.32 & 10.16
 \\\midrule
 Time (s) & 976 & 1560 & 2171 & 2868 & 3498 \\ 
\bottomrule
 \end{tabularx}}
\caption{Sensitivity analysis of MC sample size ($M$) based on BigGAN trained on CIFAR-10 dataset. The training time (seconds) is computed for 5k iterations.}\label{tab:mc_sensitivity}
\end{table}

A sweet spot around $M=10$ is observed with the best FID and a good IS. A larger MC size of fake samples
increases the variability in the IS score but causes the generated distribution to diverge from the real data.
The optimal size may vary for different datasets, indicating the need (limitation) for fine-tuning. Future work could explore its adaptive strategies.

In Table \ref{tab:mc_sensitivity}, we also report the training time comparison between the baseline and our MCGAN with varying $M$ for our BigGAN CIFAR-10 experiments. In general, the training time of MCGAN is approximately linear w.r.t Monte Carlo size $M$ for the fixed number of epochs. Since a moderate $M$ achieves satisfactory results and MCGAN often converges faster, the training time remains manageable compared to the baseline.
%optimising the expectation discriminator using the generator loss 

%\input{numerical_results/BigGAN_CIFAR10_trainingtime}

%\subsubsection{Training time}

\subsubsection{Failure case study on FFHQ256}
To evaluate the quality of generated samples, we compared the baseline model (StyleGAN2-ada) with our proposed method (StyleGAN2-ada+MC). We generated 480 samples from each model for a qualitative analysis. Our observations reveal that only the baseline model produces samples with missing facial components, as illustrated in Figure \ref{fig:subfig1}. This suggests that our method captures facial structures more effectively. Furthermore, generating coherent faces under occlusion is challenging. As shown in Figures \ref{fig:subfig2} and \ref{fig:subfig3} , our method produces more realistic facial structures behind the microphone, which we attribute to the strong supervision provided by our innovative loss function.

\begin{figure}[!ht]
    \centering
    % First subfigure on the first line
    \begin{subfigure}{\linewidth}
        \centering
        \includegraphics[width=\textwidth]{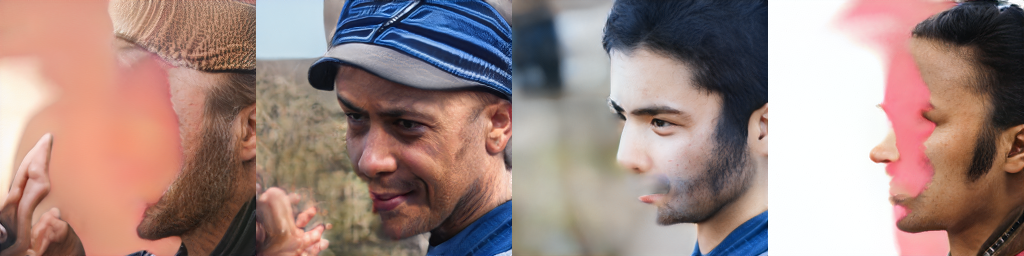}% Replace with your image file
        \caption{Samples with missing facial features generated by StyleGAN2-ada;}
        \label{fig:subfig1}
    \end{subfigure}
    
    % Second subfigure on the second line
    \begin{subfigure}{\linewidth}
        \centering
        \includegraphics[width=\textwidth]{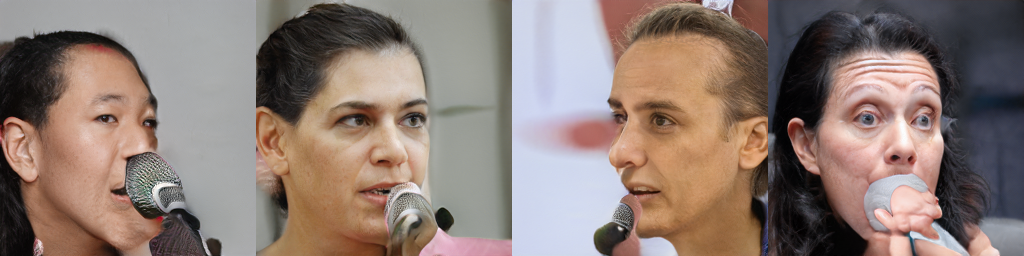} % Replace with your image file
        \caption{Samples with mic generated by StyleGAN2-ada;}
        \label{fig:subfig2}
    \end{subfigure}
    \begin{subfigure}{\linewidth}
        \centering
        \includegraphics[width=\textwidth]{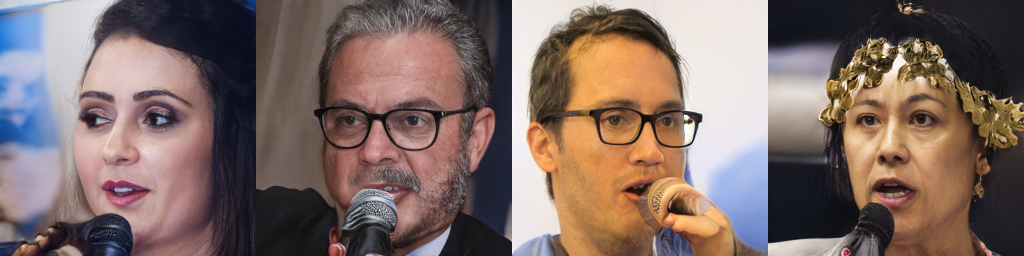} % Replace with your image file
        \caption{Samples with mic generated by StyleGAN2-ada + MC.}
        \label{fig:subfig3}
    \end{subfigure}
    \caption{Failure cases generated by baseline and our model}
    \label{fig:ffhq_badcase}
\end{figure}

\subsection{Timeseries generation} In this section, we present our numerical results on conditional timeseries generation .
\subsubsection{VAR dataset}
For VAR dataset, the evaluation metrics are give in Tables \ref{table_Var_dim=1}, \ref{table_Var_dim=2} and \ref{table_Var_dim=3}. From these tables, we can see that our MCGAN has considerable improvement over RCGAN across different parameter settings and dimensions. Also, the long-term ACF shown in Figure \ref{fig:long_acf_VAR_MCGAN} illustrates that MCGAN can better capture the temporal dependence than RCGAN.

\begin{table}[!ht]
\caption{Numerical results of $\operatorname{VAR}(1)$ for $d = 1$}\label{table_Var_dim=1}
\centering 
\resizebox{\linewidth}{!}{
\begin{tabularx}{0.8\textwidth}{l*3{|Y}}
\toprule
\multicolumn{1}{l|}{}     & \multicolumn{3}{c}{Temporal Correlations}                                                       \\ \midrule
Settings                  & \multicolumn{1}{c}{$\phi=0.2$} & \multicolumn{1}{c}{$\phi=0.5$} & \multicolumn{1}{c}{$\phi=0.8$} \\ \midrule
\multicolumn{4}{c}{Metric on marginal distribution}                                                                                              \\ \midrule
SigCWGAN&0.00522&0.00610&\textbf{0.00381}\\
MCGAN&\textbf{0.00402}&\textbf{0.00501}&0.00384\\
TimeGAN&0.0259&0.02735&0.01691\\
RCGAN&0.00443&0.00683&0.00464\\
GMMN&0.00678&0.00659&0.00554\\
\midrule
\multicolumn{4}{c}{Absolute difference of lag-1 autocorrelation}                                                                                                \\ \midrule
SigCWGAN&0.00947&0.01464&\textbf{0.00182}\\
MCGAN&0.00648&0.02047&0.00324\\
TimeGAN&0.04269&0.04526&0.01651\\
RCGAN&\textbf{0.00266}&0.01943&0.00531\\
GMMN&0.01232&\textbf{0.00106}&0.00618\\
 \midrule
\multicolumn{4}{c}{Relative $R^2$ error (\%)}                                        \\ \midrule
SigCWGAN&0.45011&\textbf{0.12953}&\textbf{0.00654}\\
MCGAN&\textbf{0.15403}&0.39642&0.03417\\
TimeGAN&7.44523&2.12036&1.38983\\
RCGAN&2.16534&0.93133&0.19214\\
GMMN&0.34882&1.36565&2.10632\\
\midrule
\multicolumn{4}{c}{Sig-$W_1$ distance}              \\ \midrule
SigCWGAN&0.69598&\textbf{1.09869}&\textbf{2.34807}\\
MCGAN&\textbf{0.69529}&1.10365&2.35118\\
TimeGAN&0.71696&1.12885&2.37692\\
RCGAN&0.69653&1.0995&2.35203\\
GMMN&0.70083&1.10592&2.3526\\
\bottomrule
\end{tabularx}
}
\end{table}

\begin{table}[!ht]
\caption{Numerical results of $\operatorname{VAR}(1)$ for $d = 2$}\label{table_Var_dim=2}
\resizebox{\linewidth}{!}{
\begin{tabularx}{0.8\textwidth}{l|Y|Y|Y||Y|Y|Y}
\toprule
\multicolumn{1}{c|}{}         & \multicolumn{3}{c||}{Temporal Correlations (fixing $\sigma=0.8$)}                                                   & \multicolumn{3}{c}{Feature Correlations (fixing $\phi=0.8$)}                                 \\ \midrule
\multicolumn{1}{c|}{Settings} & \multicolumn{1}{c|}{$\phi=0.2$}      & \multicolumn{1}{c|}{$\phi=0.5$}      & \multicolumn{1}{c||}{$\phi=0.8$}      & \multicolumn{1}{c|}{$\sigma=0.2$}    & \multicolumn{1}{c|}{$\sigma=0.5$}    & $\sigma=0.8$    \\ \midrule
\multicolumn{7}{c}{Metric on marginal distribution}            \\ \midrule
\multicolumn{1}{l|}{SigCWGAN}&0.01177&\textbf{0.00537}&\textbf{0.00365}&\textbf{0.00383}&\textbf{0.00277}&\textbf{0.00365}\\
\multicolumn{1}{l|}{MCGAN}&\textbf{0.00384}&0.00651&0.00538&0.00457&0.00502&0.00538\\
\multicolumn{1}{l|}{TimeGAN}&0.02059&0.02187&0.01113&0.00933&0.01099&0.01113\\
\multicolumn{1}{l|}{RCGAN}&0.00613&0.00706&0.00466&0.00607&0.00886&0.00466\\
\multicolumn{1}{l|}{GMMN}&0.00861&0.00912&0.00601&0.00474&0.00476&0.00601\\

\midrule

\multicolumn{7}{c}{Absolute difference of lag-1 autocorrelation}                                                                                                 \\ \midrule
\multicolumn{1}{l|}{SigCWGAN}&\textbf{0.00658}&\textbf{0.00248}&\textbf{0.00419}&\textbf{0.00353}&0.00555&\textbf{0.00419}\\	
\multicolumn{1}{l|}{MCGAN}&0.02137&0.04051&0.00716&0.00438&0.00840&0.00716\\	
\multicolumn{1}{l|}{TimeGAN}&0.04433&0.04567&0.00822&0.02446&\textbf{0.00442}&0.00822\\	
\multicolumn{1}{l|}{RCGAN}&0.01857&0.04249&0.03218&0.01227&0.03571&0.03218\\	
\multicolumn{1}{l|}{GMMN}&0.00699&0.02081&0.04263&0.08085&0.05893&0.04263\\	

\midrule
\multicolumn{7}{c}{$L_1$-norm of real and generated cross-correlation matrices}                                           \\ \midrule
\multicolumn{1}{l|}{SigCWGAN}&0.00804&0.01113&\textbf{0.01122}&\textbf{0.00476}&0.01198&\textbf{0.01122}\\	
\multicolumn{1}{l|}{MCGAN}&0.02653&0.01502&0.01149&0.01381&0.03842&0.01149\\	
\multicolumn{1}{l|}{TimeGAN}&0.08622&0.07002&0.07494&0.07455&0.04685&0.07494\\	
\multicolumn{1}{l|}{RCGAN}&0.01200&0.02846&0.03460&0.08187&0.03317&0.03460\\	
\multicolumn{1}{l|}{GMMN}&\textbf{0.00745}&\textbf{0.00565}&0.02705&0.00973&\textbf{0.00917}&0.02705\\	

\midrule
\multicolumn{7}{c}{Relative $R^2$ error (\%).}\\ \midrule
\multicolumn{1}{l|}{SigCWGAN}&\textbf{1.24036}&\textbf{0.09027}&\textbf{0.01252}&\textbf{0.01381}&\textbf{0.01248}&\textbf{0.01252}\\
\multicolumn{1}{l|}{MCGAN}&10.24923&1.61850&0.42136&0.26449&0.35319&0.42136\\
\multicolumn{1}{l|}{TimeGAN}&40.1273&4.92783&1.21018&1.05100&0.89636&1.21018\\
\multicolumn{1}{l|}{RCGAN}&18.33682&4.31191&1.39435&3.94201&1.58417&1.39435\\
\multicolumn{1}{l|}{GMMN}&35.25094&15.76457&6.56956&12.42385&9.88914&6.56956\\
\midrule
\multicolumn{7}{c}{Sig-$W_1$ distance}                         \\ \midrule
\multicolumn{1}{l|}{SigCWGAN}&\textbf{1.92823}&2.42590&\textbf{3.60068}&\textbf{3.02208}&3.23497&\textbf{3.60068}\\	
\multicolumn{1}{l|}{MCGAN}&1.93087&\textbf{2.42466}&3.61617&3.02879&3.2390&3.61617\\	
\multicolumn{1}{l|}{TimeGAN}&1.98070&2.47622&3.63571&3.04472&3.26746&3.63571\\	
\multicolumn{1}{l|}{RCGAN}&1.93333&2.43379&3.61464&3.03564&\textbf{3.21083}&3.61464\\	
\multicolumn{1}{l|}{GMMN}&1.94517&2.43949&3.60922&3.02910&3.23898&3.60922\\	

% \multicolumn{2}{c}{Time for one epoch (sec)}                                                                                                                        \\ \midrule
% \multicolumn{1}{l|}{SigCWGAN}  & 4.58 \\
% \multicolumn{1}{l|}{CWGAN}  &  6.94 \\
% \multicolumn{1}{l|}{TimeGAN} & 3.09 \\
% \multicolumn{1}{l|}{RCGAN}  & 2.97 \\
% \multicolumn{1}{l|}{GMMN}  & 1.65 \\ \bottomrule
\bottomrule
\end{tabularx}
}
\end{table}

\begin{table}[!ht]
\caption{Numerical results of $\operatorname{VAR}(1)$ for $d = 3$}\label{table_Var_dim=3}
\centering
\resizebox{\linewidth}{!}{
\begin{tabularx}{0.8\textwidth}{l|Y|Y|Y||Y|Y|Y}
\toprule
\multicolumn{1}{c|}{}         & \multicolumn{3}{c||}{Temporal Correlations (fixing $\sigma=0.8$)} & \multicolumn{3}{c}{Feature Correlations (fixing $\phi=0.8$)} \\ \midrule
\multicolumn{1}{c|}{Settings} & $\phi=0.2$          & $\phi=0.5$          & $\phi=0.8$          & $\sigma=0.2$       & $\sigma=0.5$       & $\sigma=0.8$       \\ \midrule
\multicolumn{7}{c}{Metric on marginal distribution}                                                                                                            \\ \midrule
\multicolumn{1}{l|}{SigCWGAN}&0.01463&0.01240&\textbf{0.00477}&\textbf{0.00423}&\textbf{0.00452}&\textbf{0.00477}\\
\multicolumn{1}{l|}{MCGAN}&\textbf{0.00476}&\textbf{0.00436}&0.00596&0.00715&0.00661&0.00596\\
\multicolumn{1}{l|}{TimeGAN}&0.02359&0.02096&0.00886&0.01054&0.00915&0.00886\\
\multicolumn{1}{l|}{RCGAN}&0.01068&0.00634&0.00577&0.00836&0.00597&0.00577\\
\multicolumn{1}{l|}{GMMN}&0.01001&0.01024&0.00987&0.01427&0.01323&0.00987\\
\midrule
\multicolumn{7}{c}{Absolute difference of lag-1 autocorrelation}                                                                                              \\ \midrule
\multicolumn{1}{l|}{SigCWGAN}&\textbf{0.00570}&\textbf{0.00508}&\textbf{0.00131}&\textbf{0.00330}&\textbf{0.00172}&\textbf{0.00131}\\
\multicolumn{1}{l|}{MCGAN}&0.00684&0.01805&0.0199&0.00947&0.00529&0.0199\\
\multicolumn{1}{l|}{TimeGAN}&0.04601&0.09309&0.01643&0.03144&0.04736&0.01643\\
\multicolumn{1}{l|}{RCGAN}&0.05663&0.04925&0.02041&0.01894&0.01863&0.02041\\
\multicolumn{1}{l|}{GMMN}&0.04041&0.06024&0.08998&0.10196&0.13395&0.08998\\
\midrule
\multicolumn{7}{c}{$L_1$-norm of real and generated cross-correlation matrices}                                                                               \\ \midrule
\multicolumn{1}{l|}{SigCWGAN}&\textbf{0.01214}&\textbf{0.01311}&\textbf{0.00317}&\textbf{0.01715}&\textbf{0.02862}&\textbf{0.00317}\\
\multicolumn{1}{l|}{MCGAN}&0.04076&0.03819&0.03659&0.04631&0.08001&0.03659\\
\multicolumn{1}{l|}{TimeGAN}&0.20056&0.43239&0.15509&0.09314&0.09228&0.15509\\
\multicolumn{1}{l|}{RCGAN}&0.24082&0.16809&0.09657&0.16257&0.11514&0.09657\\
\multicolumn{1}{l|}{GMMN}&0.09850&0.12638&0.20142&0.3096&0.37507&0.20142\\
\midrule
\multicolumn{7}{c}{Relative $R^2$ error (\%).}                                                                                               \\ \midrule
\multicolumn{1}{l|}{SigCWGAN}&\textbf{1.393190}&\textbf{0.34009}&\textbf{0.07690}&\textbf{0.05323}&\textbf{0.03498}&\textbf{0.07690}\\
\multicolumn{1}{l|}{MCGAN}&14.63272&1.62564&0.56412&0.32549&0.42388&0.56412\\
\multicolumn{1}{l|}{TimeGAN}&36.71498&8.94899&2.38110&2.61944&3.80723&2.38110\\
\multicolumn{1}{l|}{RCGAN}&70.69909&16.50512&2.83140&1.44543&2.79532&2.83140\\
\multicolumn{1}{l|}{GMMN}&152.87792&38.97992&17.94085&25.12542&26.93346&17.94085\\
\midrule
\multicolumn{7}{c}{Sig-$W_1$ distance}                                                                                                                        \\ \midrule
\multicolumn{1}{l|}{SigCWGAN}&11.57390&\textbf{17.66105}&30.46722&30.75008&25.24824&30.46722\\
\multicolumn{1}{l|}{MCGAN}&11.57797&17.69298&30.48571&\textbf{30.73360}&\textbf{25.22319}&30.48571\\
\multicolumn{1}{l|}{TimeGAN}&11.88320&18.09083&30.70047&30.86857&25.36035&30.70047\\
\multicolumn{1}{l|}{RCGAN}&\textbf{11.53368}&17.72101&\textbf{30.40070}&30.74105&25.30295&\textbf{30.40070}\\
\multicolumn{1}{l|}{GMMN}&11.61313&17.73444&30.59544&30.79028&25.38754&30.59544\\
\midrule
\bottomrule
\end{tabularx}
}
\end{table}

 \begin{figure}[!ht]
        \centering
        \includegraphics[width=1\linewidth]{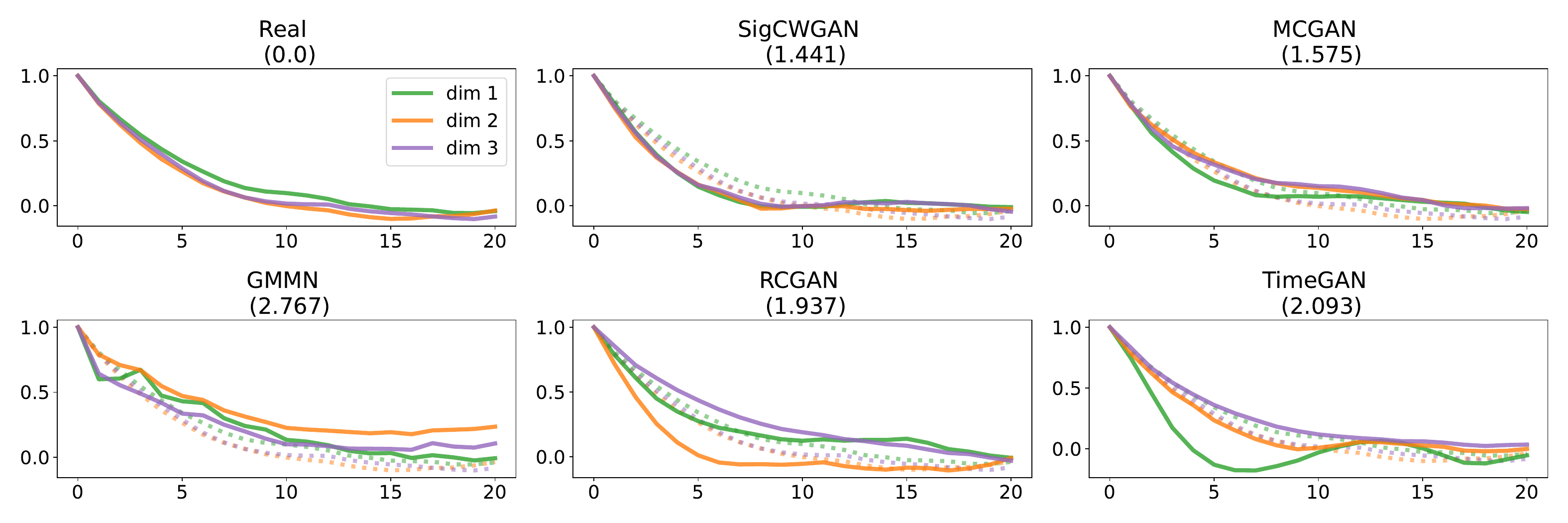}
        \caption{ACF plot for each channel on the 3-dimensional VAR(1) dataset  with autocorrelation coefficient
$\phi = 0.8$ and co-variance parameter $\sigma = 0.8$. Here $x$-axis represents the lag value ( with a maximum lag equal to 100) and the $y$-axis represents the corresponding auto-correlation. The length of the real/generated time series used to compute the ACF is 1000. The number in the bracket under each model is the sum of the absolute difference between the correlation coefficients computed from real (dashed line) and generated (solid line) samples.}
        \label{fig:long_acf_VAR_MCGAN}
    \end{figure}

\subsubsection{Stocks dataset} For stock dataset, we generate both log return and log volatility process of S\&P 500 and DJI . The estimated cross-corelation matrices are presented in Figure \ref{Fig:Correlation_Comparison_SPX}, we can see that the generated cross-corelation matrices by MCGAN are closer to the ground truth comparing with that of RCGAN, indicating the effectiveness of our MC method in  capturing cross dependency.

\setcounter{figure}{9}
\begin{figure}[!ht]
\centering
\includegraphics[width=1\linewidth]{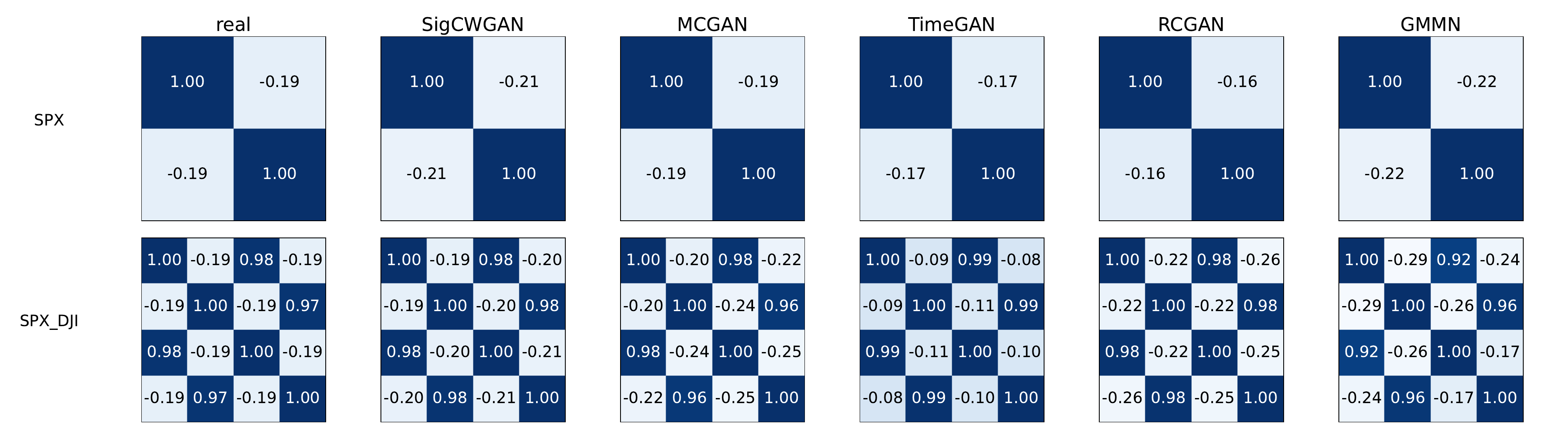}
\caption{Comparison of real and synthetic cross-correlation matrices for SPX/ SPX and DJI data. On the far left the real cross-correlation matrix from SPX/ SPX and DJI log-return and log-volatility data is shown. $x$/$y$-axis represents the feature dimension while the color of the $(i, j)^{th}$ block represents the correlation of $X_{t}^{(i)}$ and $X_{t}^{(j)}$. Observe that the historical correlation between log returns and log volatility is negative, indicating the presence of leverage effects, i.e. when log returns are negative, log volatility is high.}\label{Fig:Correlation_Comparison_SPX}
\end{figure}

\subsection{Synthetic 2D dataset}
We conducted experiments using 2D grid-like synthetic data with 25 Gaussian modes and 5,000 generated samples. Our MCGAN registered all 25 modes with Total Variation TV=$14.64\pm5.50$, outperforming vanilla GAN (17 modes, TV=$35.23\pm2.02$) and LSGAN (20 modes, TV=$29.72\pm6.40$), indicating our ability to alleviate mode collapse. Increasing MC size to 50/100 reduced TV further to $6.99\pm4.67$/$3.54\pm2.17$ as shown in Table \ref{tab:synthetic_results}. Visualizations are given in Figure \ref{fig:gaussianmixture}.

\begin{table}[!ht]
\caption{Results on 2D synthetic data. Test metrics are computed using 5000 generated samples for 10 different seeds. The label MC=$n$ indicates that MCGAN is used with a Monte Carlo sample size of $M=n$.}
\centering
\resizebox{\linewidth}{!}{
\begin{tabularx}{0.7\textwidth}{l*3{|Y}}
\hline
\textbf{Method} & \textbf{\# Registered Modes} & \textbf{\# Registered Points} & \textbf{Total Variation} \\ \hline
GAN    & $17.4 \pm 3.1$  & $4511.8 \pm 63.86$  & $35.23 \pm 2.02$   \\ \hline
LSGAN    & $20.4 \pm 1.2$    & $4464.2 \pm 182.85$ & $29.72 \pm 6.40$  \\ \hline
MC =10    & $25 \pm 0.0$      & $4659.4 \pm 75.72$   & $14.64 \pm 5.50$   \\ \hline
MC=50    & $25 \pm 0.0$      & $\mathbf{4807.8 \pm 21.07}$  & $6.99 \pm 4.68$   \\ \hline
MC=100   & $\mathbf{25 \pm 0.0}$     & $4800.4 \pm 59.85$   & $\mathbf{3.54 \pm 2.17}$   \\ \hline
\end{tabularx}
}
\label{tab:synthetic_results}
\end{table}

\begin{figure}[!ht]
	\centering
	\begin{subfigure}{0.45\linewidth}
		\centering
		\includegraphics[width=0.9\textwidth]{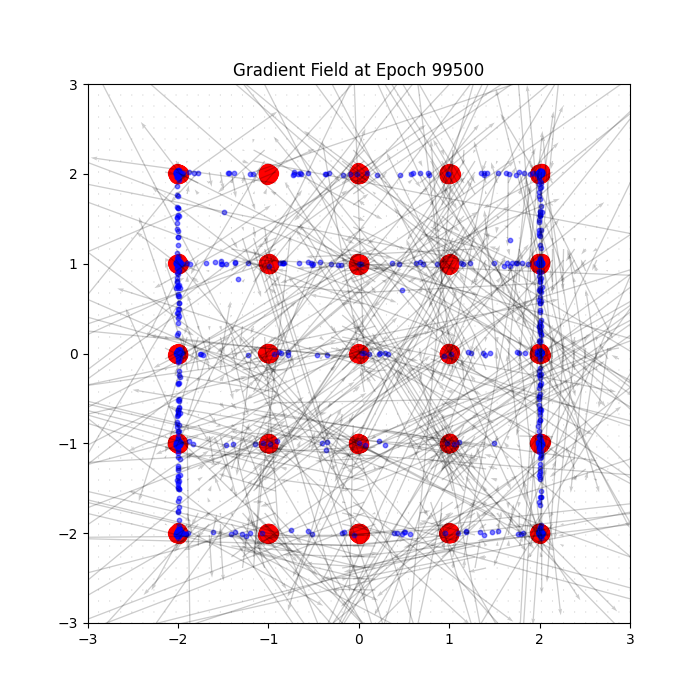}
		\caption{Vanilla GAN (TV=38.62)}
	\end{subfigure}
        \begin{subfigure}{0.45\linewidth}
		\centering
		\includegraphics[width=0.9\textwidth]{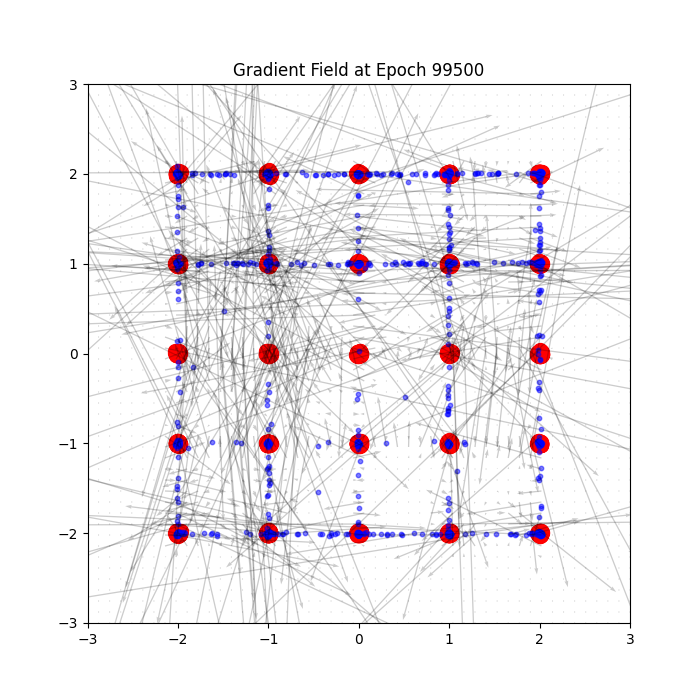}
		\caption{LSGAN (TV=24.97)}
	\end{subfigure}
 	\begin{subfigure}{0.45\linewidth}
		\centering
		\includegraphics[width=0.9\textwidth]{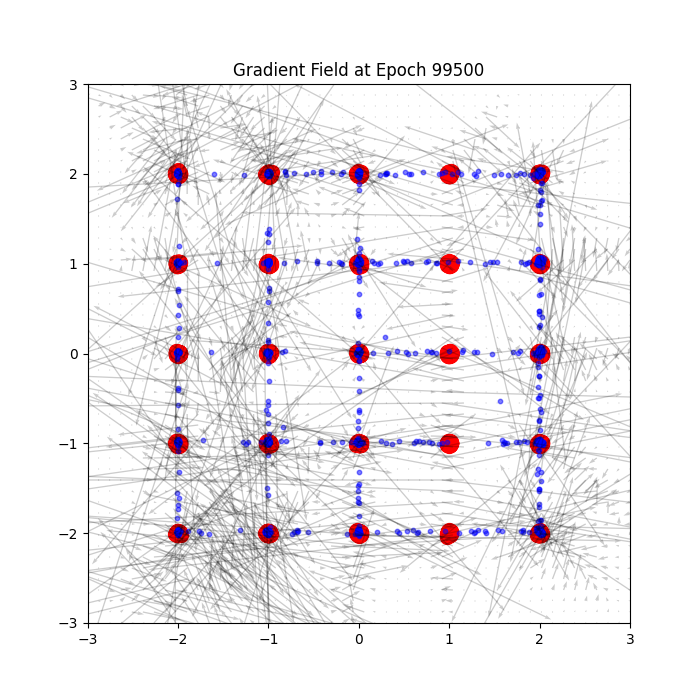}
		\caption{MC=10 (TV=14.40)}
	\end{subfigure}
	\begin{subfigure}{0.45\linewidth}
		\centering
		\includegraphics[width=0.9\textwidth]{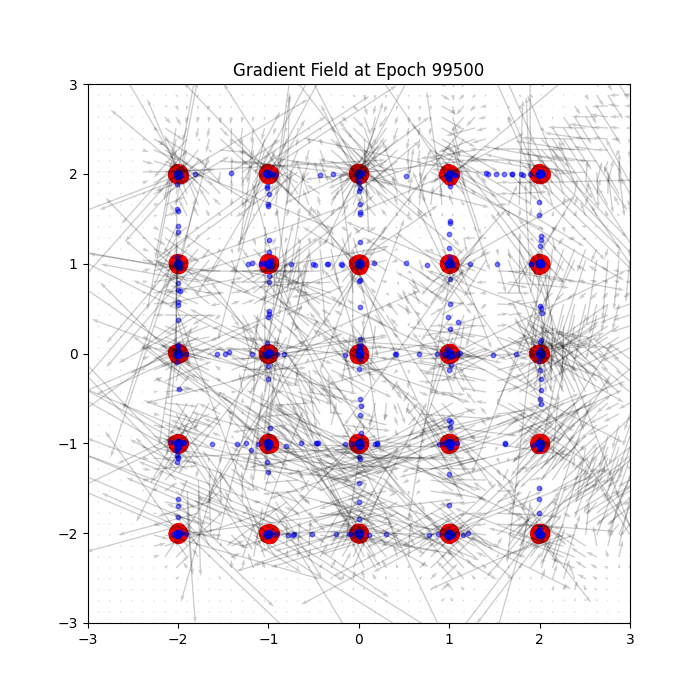}
		\caption{MC=100 (TV=2.84)}
	\end{subfigure}
\caption{Example of generated samples by different methods; Red points are 5000 real samples with 0.01 standard deviation;  Blue points are 5000 generated samples; Dash lines illustrate gradients of discriminator on each point.}\label{fig:gaussianmixture}
\end{figure}

\section{Generated Samples}
In this section, we show some generated samples generated by our MCGAN models including images, time-series.

\subsection{Generated CIFAR-10 samples by BigGAN backbone}
Here we present the samples generated by BigGAN backbone in Figure \ref{fig:cifar10_biggan}. We can see that only a few generated figures are misclassified, showcasing the ability of MCGAN in genrating high-fidility samples.

\subsection{Generated CIFAR-10 samples by StyleGAN2 backbone}
Here we present the samples generated by StyleGAN2 backbone in Figure \ref{fig:stylegan2}. Comparing with Figure \ref{fig:cifar10_biggan}, fewer generated samples are misclassified due to the employment of a much stronger backbone.
\setcounter{figure}{8}
\begin{figure}[!ht]
    \centering
\includegraphics[width=\linewidth]{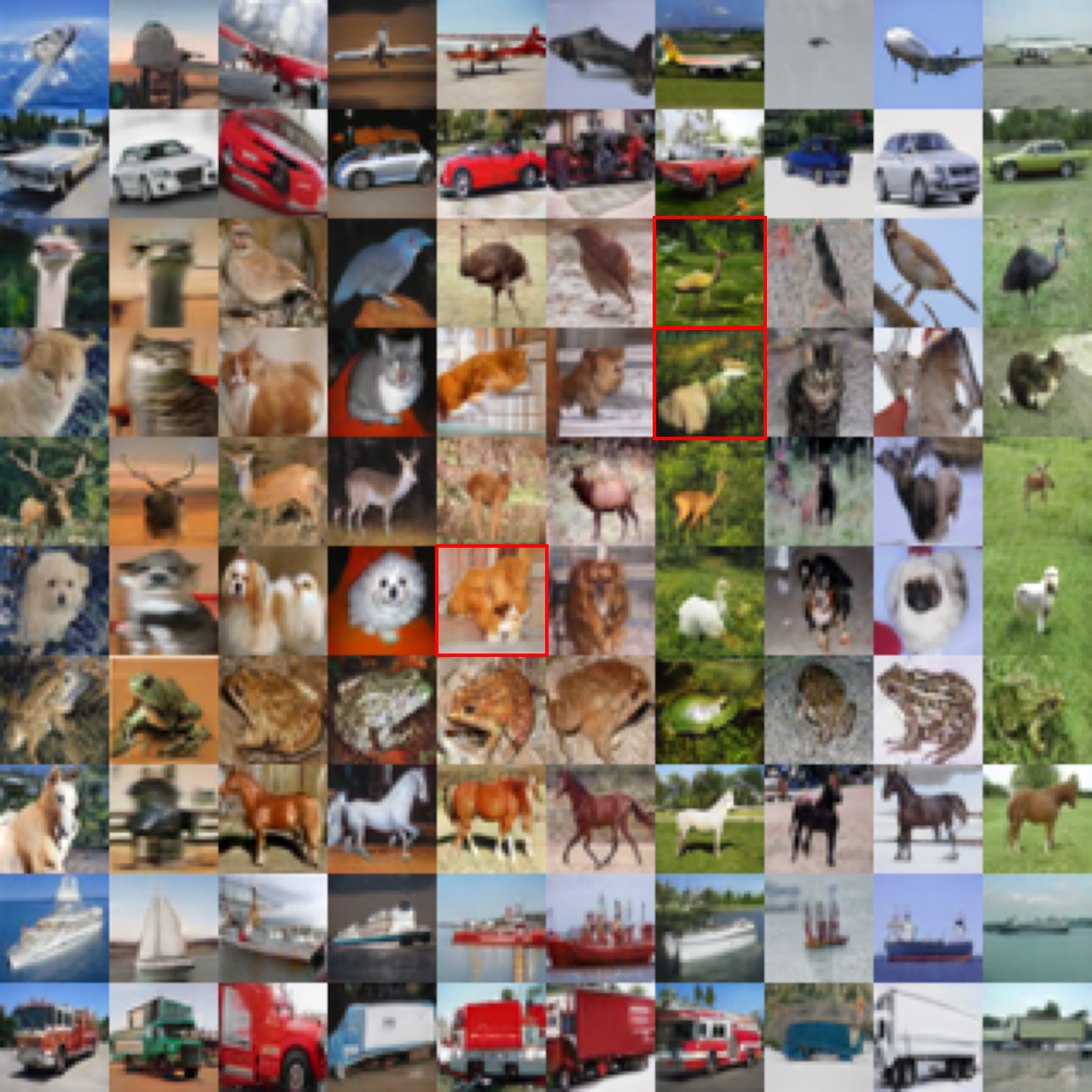}
    \caption{CIFAR-10 samples generated by the cStyleGAN2 backbone trained via Hinge + DiffAug + MC. Images in each row belong to one of the 10 classes. Images misclassified by ResNet-50 are in red boxes.}
    \label{fig:stylegan2}
\end{figure}

\subsection{Generated FFHQ256 samples by StyleGAN2 backbone}
In this subsection, we present some FFHQ256 samples generated by StylGAN2 trained via our MC methods in Figure \ref{fig:ffhq}. These 64 images are randomly picked out of 480 generated samples. We can observe from Figure \ref{fig:ffhq} human faces with different skins, ages, angles, lighting and accessories, showcasing that StyleGAN2 trained via our MC method has the ability to generate realistic, diversified, and  high-resolution human face images.
\setcounter{figure}{10}
\begin{figure}[!ht]
    \centering
\includegraphics[width=\linewidth]{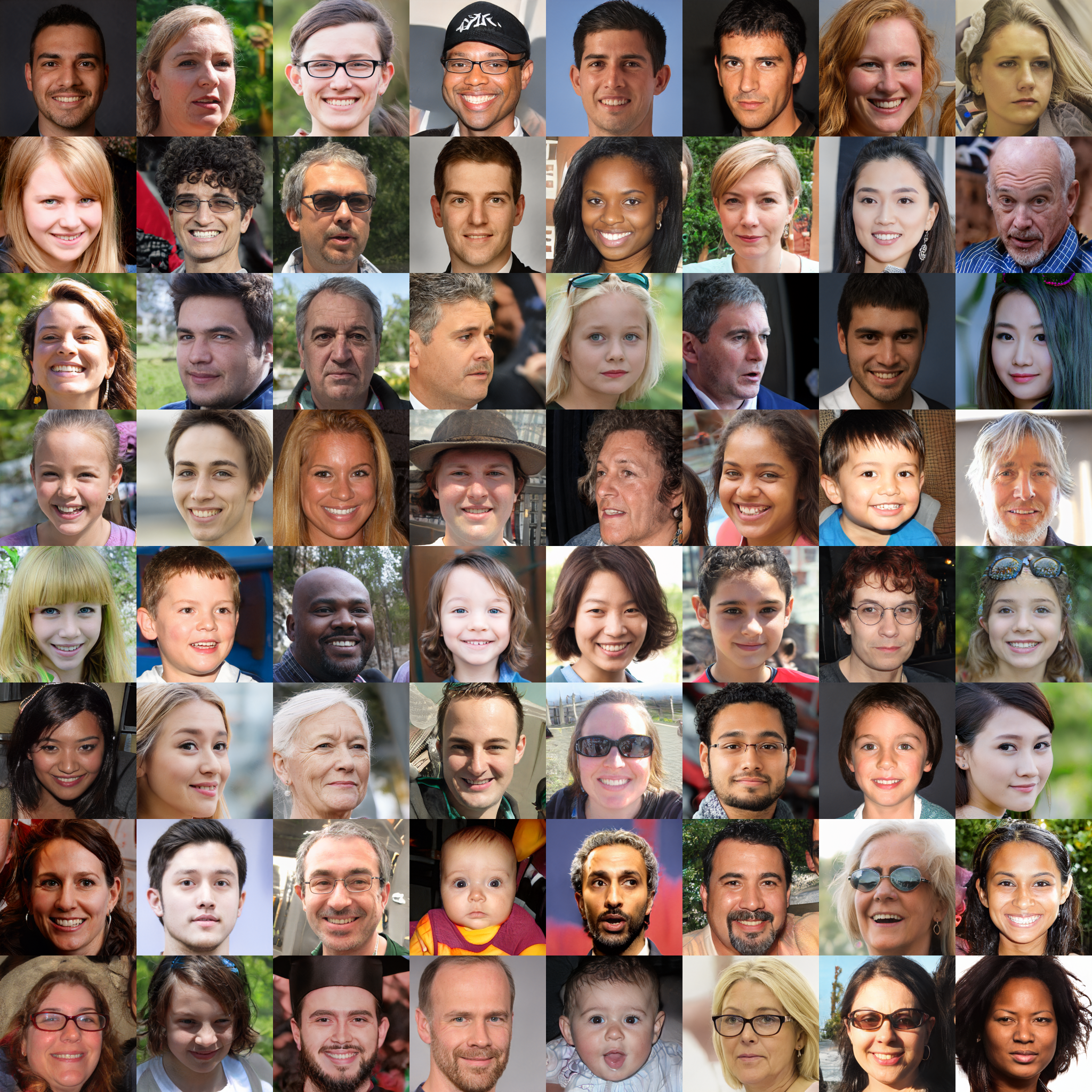}
    \caption{FFHQ256 samples generated by the StyleGAN2 backbone trained via our MC method with FID $3.77\pm 0.04$.}
    \label{fig:ffhq}
\end{figure}

\subsection{Generated ImageNet64 samples by StyleGAN2 backbone}
In this subsection, we present 100 samples generated by cStylGAN2 backbone trained via our MC methods in Figure \ref{fig:imagenet64}. Due to the large scale of ImageNet64 dataset, it is a rather challenging conditional generation
 task.

\begin{figure}[!ht]
    \centering
\includegraphics[width=1\linewidth]{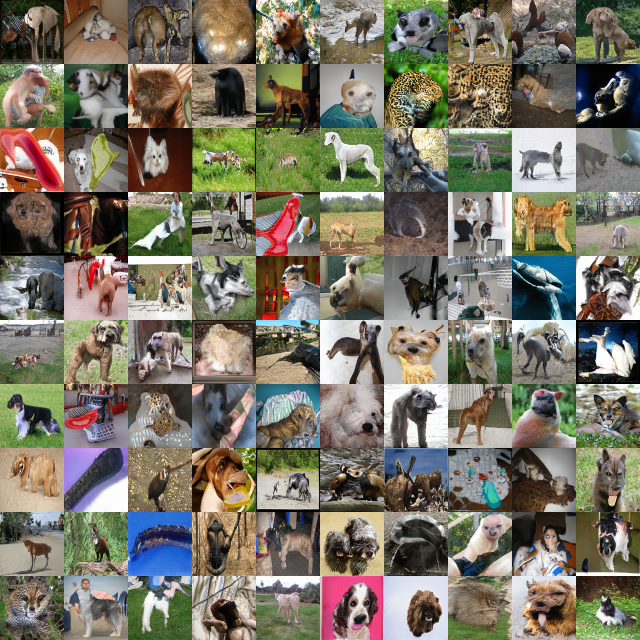}
    \caption{100 ImageNet64 samples generated by the cStyleGAN2 backbone trained via our MC method with FID $16.76\pm0.08$.}
    \label{fig:imagenet64}
\end{figure}

\subsection{Generated LSUN bedroom samples by StyleGAN2 backbone}
In this subsection, we present 100 samples generated by cStylGAN2 backbone trained via our MC methods in Figure \ref{fig:lsunbed}.

\begin{figure}[!ht]
    \centering
\includegraphics[width=1\linewidth]{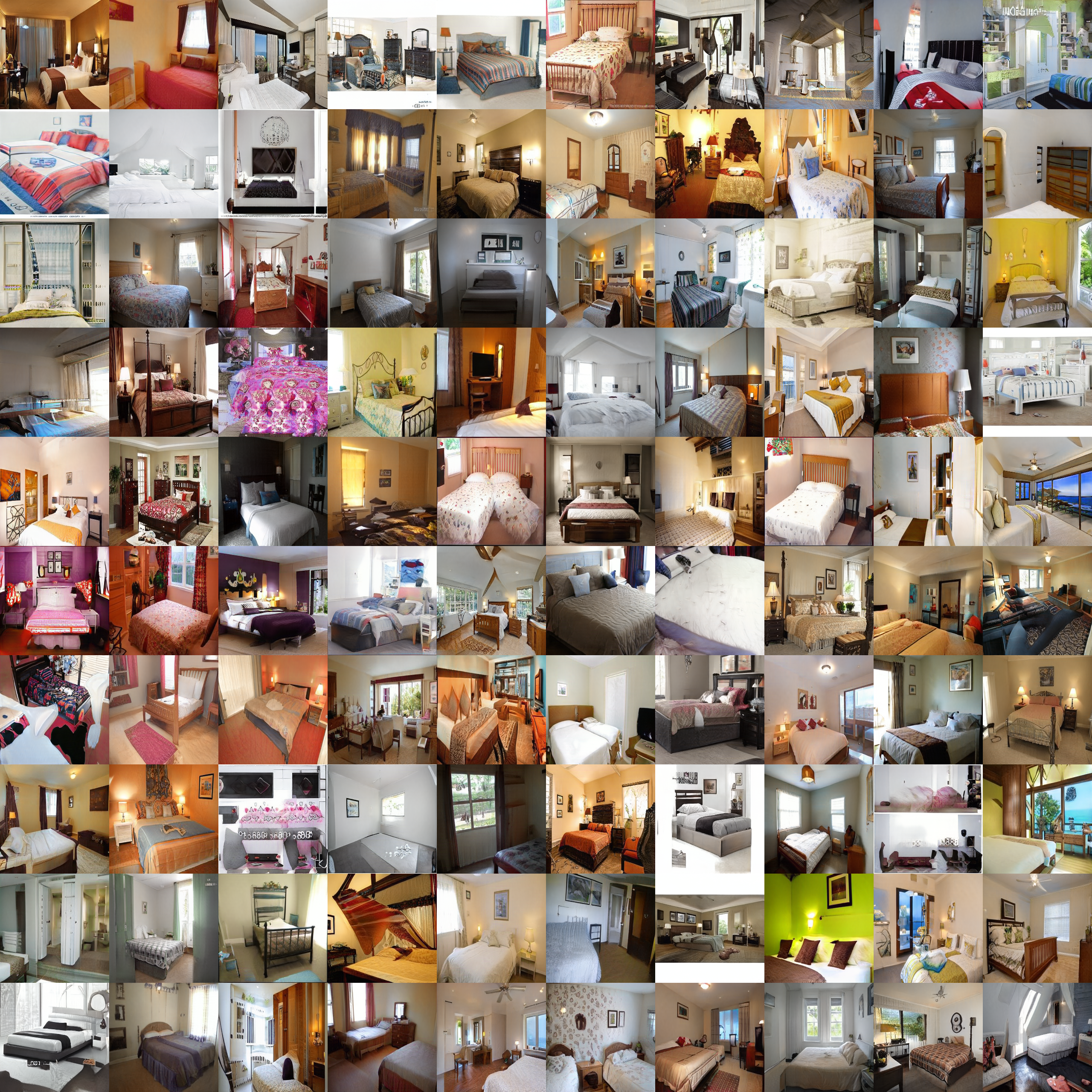}
    \caption{100 LSUN bedroom samples generated by the cStyleGAN2 backbone trained via our MC method with FID $2.77\pm0.03$.}
    \label{fig:lsunbed}
\end{figure}

\subsection{Generated stock data samples by RNN backbone}
In this subsection, we present generated SPX log-return paths of each model in Figure \ref{fig:examplepath}. We can see that visually the log-return path generated by MCGAN shows volatility clustering property and looks closer to the historical path than that of RCGAN.  We also provide comparison of marginal distributions of ground truth and generated paths of MCGAN in Figure \ref{fig:mcgan_stock_density}. We can see that the generated histogram is very close to the historical one in terms of mean, standard deviation, skewness and kurtosis.

\begin{figure}[!ht]
	\centering
	\begin{subfigure}{0.45\linewidth}
		\centering
		\includegraphics[width=0.9\textwidth]{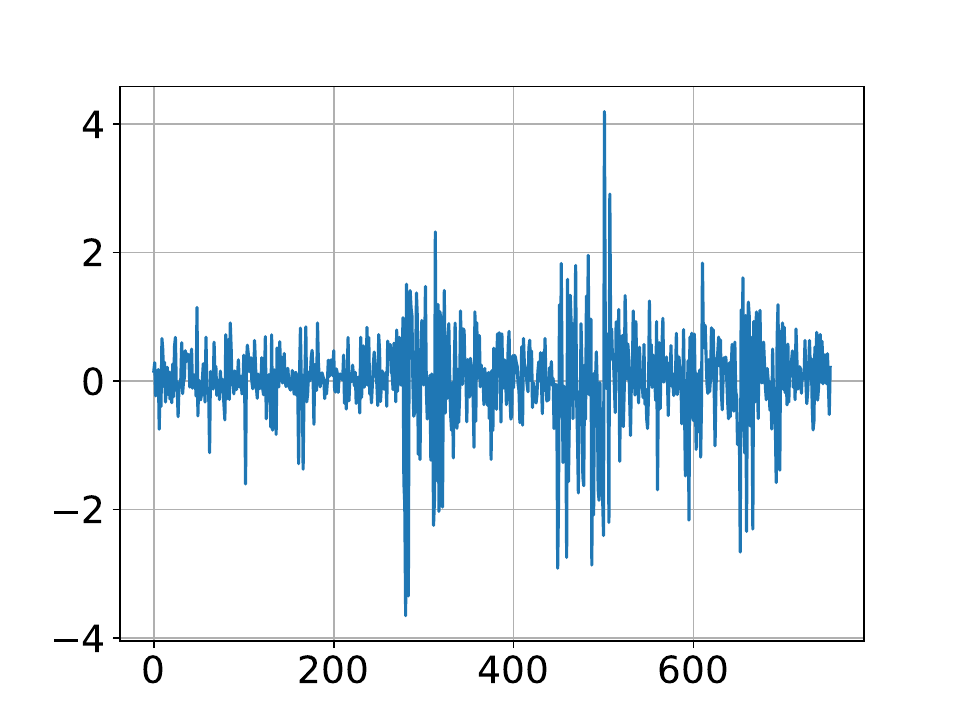}
		\caption{Real}
	\end{subfigure}
        \begin{subfigure}{0.45\linewidth}
		\centering
		\includegraphics[width=0.9\textwidth]{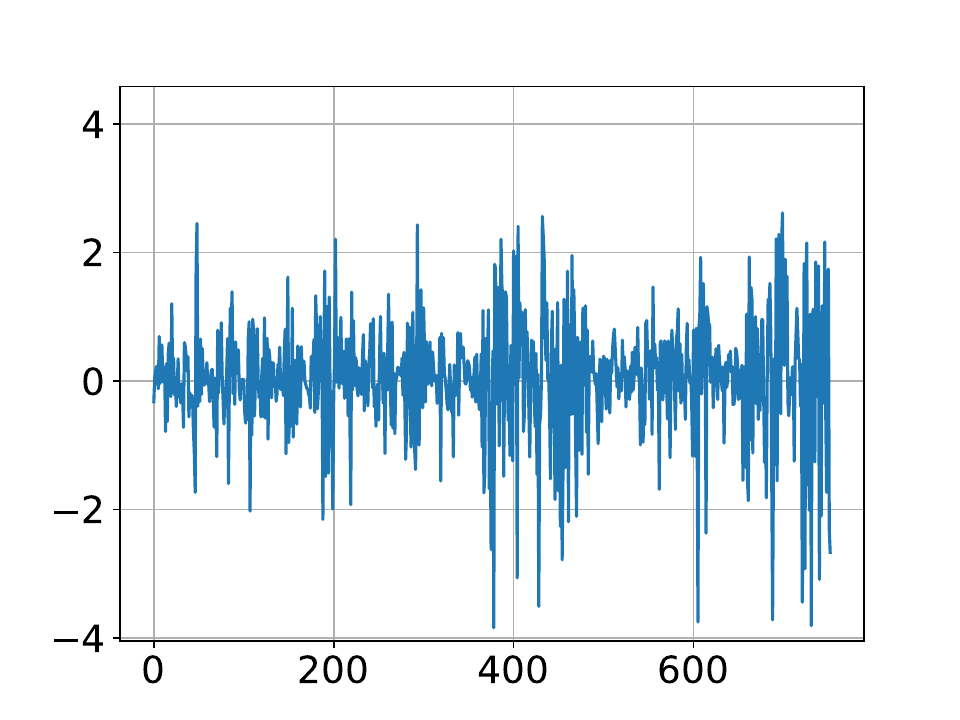}
		\caption{MCGAN}
	\end{subfigure}
 	\begin{subfigure}{0.45\linewidth}
		\centering
		\includegraphics[width=0.9\textwidth]{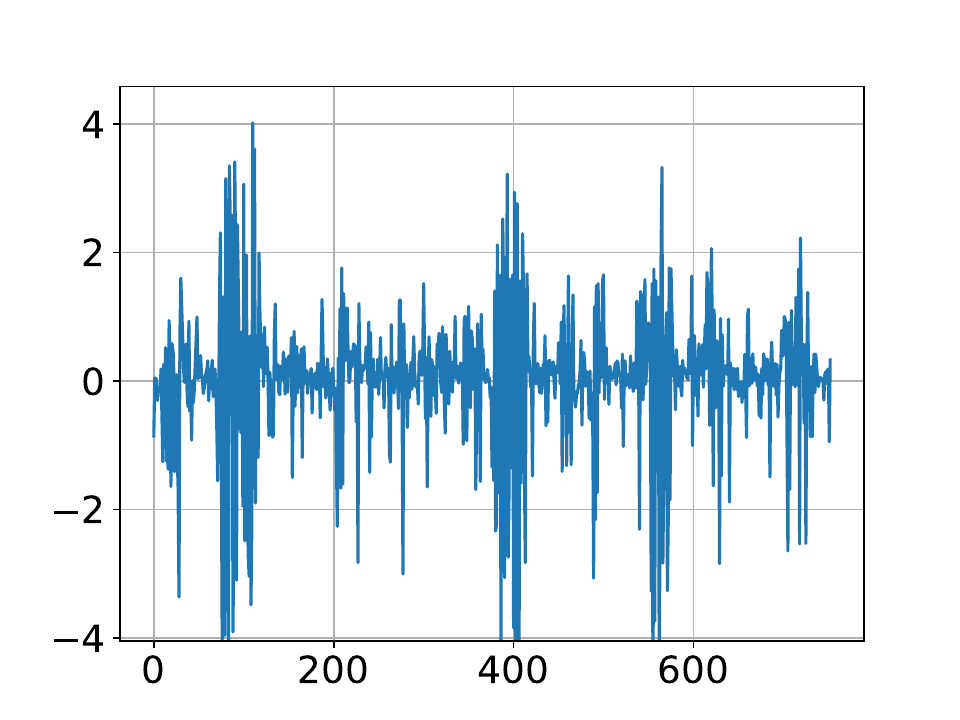}
		\caption{SigCWGAN}
	\end{subfigure}
	\begin{subfigure}{0.45\linewidth}
		\centering
		\includegraphics[width=0.9\textwidth]{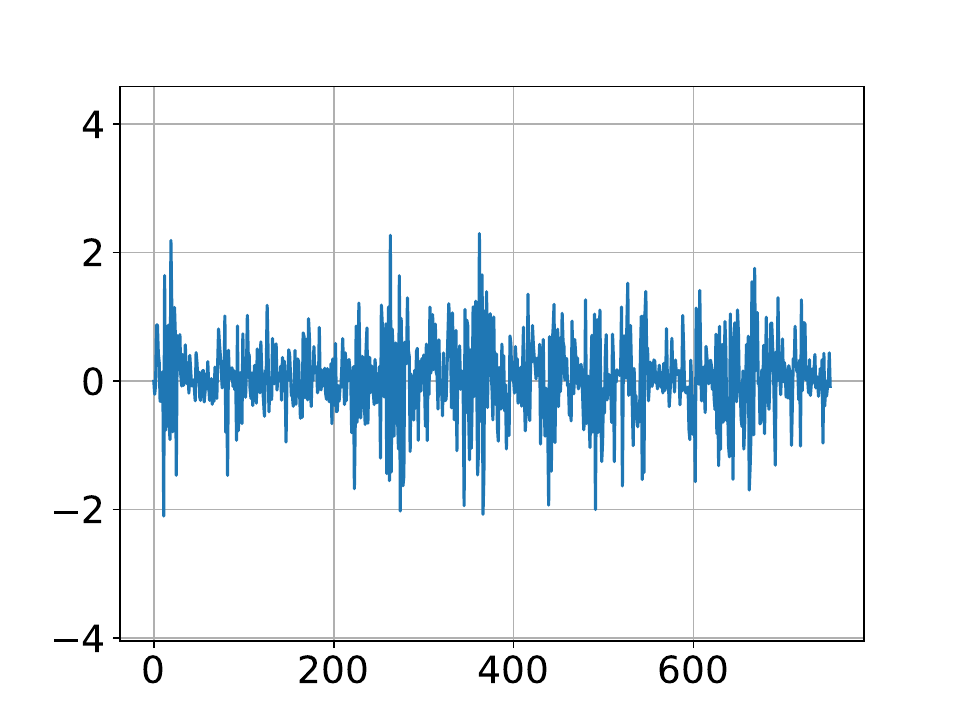}
		\caption{TimeGAN}
	\end{subfigure}
	\begin{subfigure}{0.45\linewidth}
		\centering
		\includegraphics[width=0.9\textwidth]{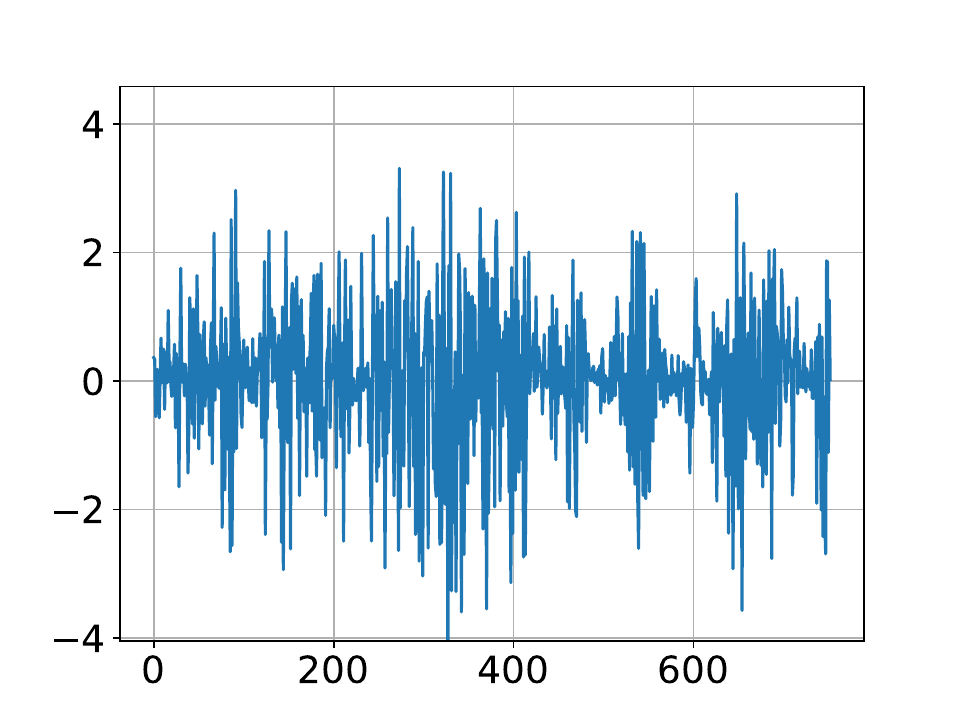}
		\caption{RCGAN}	
	\end{subfigure}
	\begin{subfigure}{0.45\linewidth}
		\centering
		\includegraphics[width=0.9\textwidth]{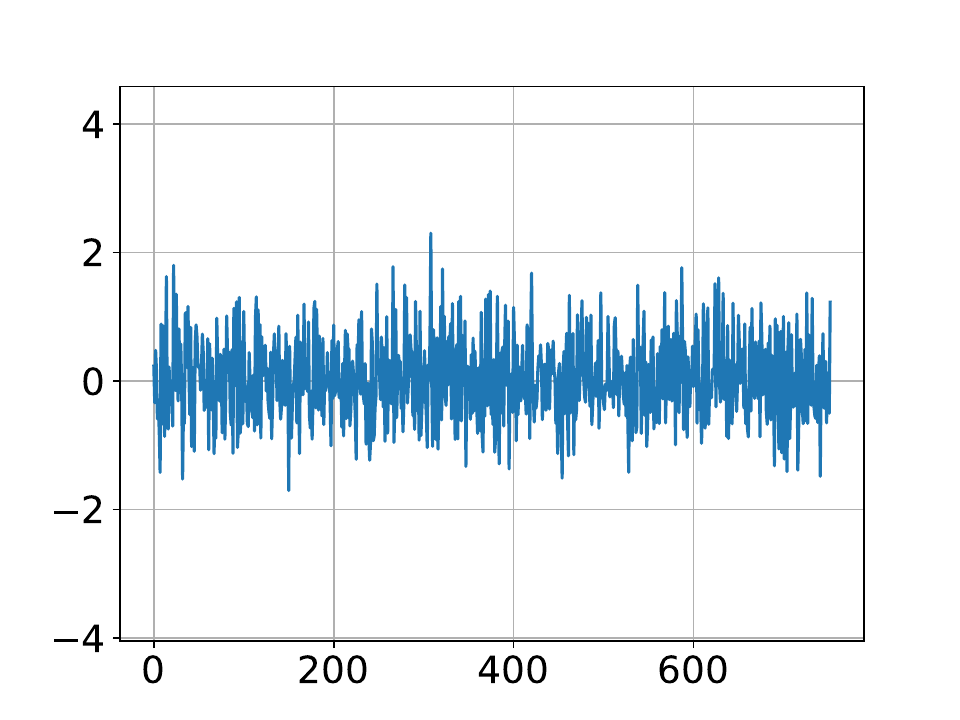}
		\caption{GMMN}
	\end{subfigure}
\caption{Example paths of SPX log returns generated by each model. Since the path of DJI log returns is similar to that of SPX, there is no need to make another plot for DJI. }\label{fig:examplepath}
\end{figure}

\begin{figure}[htb]
\centering % <-- added
\includegraphics[width=\linewidth]{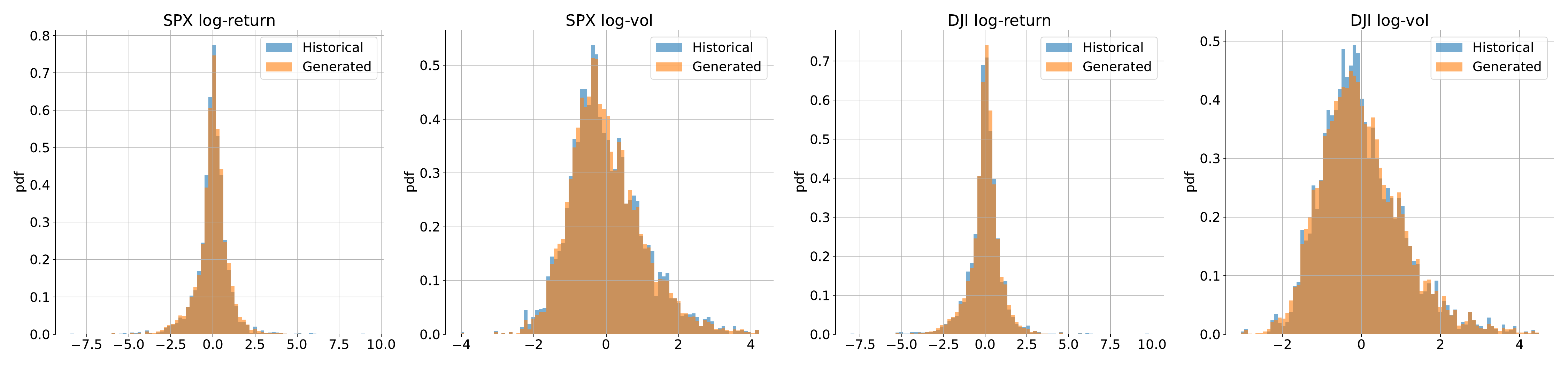}
\caption{Comparison of the marginal distributions of the  paths generated by MCGAN and the SPX and DJI data.}
\label{fig:mcgan_stock_density}
\end{figure}

\end{document}